\documentclass[11pt]{article}

\usepackage[a4paper,margin=1in]{geometry}
\usepackage[T1]{fontenc}
\usepackage{lmodern}
\usepackage{amsmath,amssymb,amsthm}
\usepackage{mathtools}
\usepackage{graphicx}
\usepackage{array}
\usepackage{booktabs}
\usepackage{float}
\usepackage{authblk}
\usepackage[dvipsnames]{xcolor}
\usepackage[most]{tcolorbox}
\usepackage{pifont}
\usepackage{microtype}
\usepackage{enumitem}
\usepackage[numbers,sort&compress]{natbib}
\usepackage{caption}
\usepackage{abstract}
\usepackage[colorlinks=true,linkcolor=NavyBlue,citecolor=ForestGreen,urlcolor=NavyBlue]{hyperref}

\newtheorem{theorem}{Theorem}
\newtheorem{proposition}{Proposition}
\newtheorem{lemma}{Lemma}
\newtheorem{corollary}{Corollary}
\theoremstyle{definition}
\newtheorem{definition}{Definition}
\newtheorem{assumption}{Assumption}
\newtheorem{postulate}{Postulate}
\theoremstyle{remark}
\newtheorem{remark}{Remark}

\newtcolorbox{hypobox}{colback=NavyBlue!4,colframe=NavyBlue!55,boxrule=1pt,arc=3pt,
  left=8pt,right=8pt,top=6pt,bottom=6pt}

\newcommand{\Tch}{\mathcal{T}}
\newcommand{\Sch}{\mathcal{S}}
\newcommand{\Hch}{\mathcal{H}}
\newcommand{\Cinf}{\mathcal{C}_\infty}
\newcommand{\Capb}{\mathrm{Cap}}
\newcommand{\Solv}{\mathrm{Solv}}
\newcommand{\Rgm}{\mathcal{R}}

\title{\bfseries The Capability Convergence Hypothesis:\\
Capability from Access Structure, Not Scale}

\author[1]{Wenhui Chen\thanks{\texttt{mc35092@um.edu.mo}}}
\author[2]{Jianlin Chen\thanks{Joint second authors, equal contribution.\ \texttt{202330450231@mail.scut.edu.cn}}}
\author[1]{Ziyao Lin\thanks{Joint second authors, equal contribution.\ \texttt{mc35081@um.edu.mo}}}
\author[1]{Chi Man Vong\thanks{Corresponding author.\ \texttt{cmvong@um.edu.mo}}}
\affil[1]{University of Macau}
\affil[2]{South China University of Technology}
\date{}

\begin{document}
\maketitle

\begin{abstract}
\noindent
The Platonic Representation Hypothesis (PRH) holds that as models scale, the internal
representations of heterogeneous networks converge toward a shared statistical model of
reality $Z$. We propose its sequel and boundary, the \textbf{Capability Convergence
Hypothesis (CCH)}: under a fixed per-token inference budget, representational convergence
does not entail capability convergence. Capability instead converges toward a \emph{class},
the \textbf{access-complete hybrid}: any form holding both a compressive state channel and
a verbatim index channel. The argument is anchored by a thought experiment playing, for
CCH, the role Plato's cave plays for PRH: the \emph{Newton's-apple problem in an infinite
stream}. Three walls name its three resource deficits: an information-theoretic
\emph{Shannon wall} barring any $o(Nb)$-state architecture, a \emph{horizon wall} barring
any fixed window, and a \emph{circuit wall} barring fixed-depth attention-only composition
(conditional on $\mathrm{TC}^0 \neq \mathrm{NC}^1$). A hybrid holding both kinds of access
crosses all three by paying each wall's price rather than evading its bound: under an
explicit separability assumption it carries the $\Omega(Nb)$ index the capacity floor
demands, defers the horizon by the coefficient $1/\rho$ (not eliminating it), and supplies
the serial composition a fixed-depth stack lacks. Capability is thereby strictly
super-additive under composition, in the set sense that the solvable-task family strictly
expands: $\Cinf(\Tch \sqcup \Sch) \supsetneq \Cinf(\Tch) \cup \Cinf(\Sch)$. The framework
rests on one object, the \emph{query-time state channel}: attention approximates
high-fidelity indexed access at cost $\Theta(L)$, state-space models realize compressed
mixing at cost $O(1)$, and the task-matched operating point lies at their union. We
separate what we prove from what we conjecture: the access-completeness principle is
established by information-theoretic lower bounds and \emph{pre-registered} experiments,
while the field-level convergence trend is an economics-motivated conjecture we mark as
such (not a measured phenomenon in PRH's sense). We give its information-theoretic skeleton
and falsification conditions, document four industrial regularities consistent with it, and
report the first pre-registered small-scale tests under criteria frozen before the data
existed. The predicted scissors gap is measured
(exact-retrieval error $0.994$ vs.\ $0.000$ once the same 64-scalar state gains one
global-attention layer, an intervention that pays the index capacity the floor demands, not
an equal-budget comparison); the state-tracking bifurcation lands at the pre-registered
boundary for the tested configuration; alignment rises with scale while capability
stratifies by access structure; and an end-to-end conjunction witness shows a
training-reliability separation, the hybrid passing the registered cell on every seed while
no other arm meets the criterion, with ablations showing the learned solution irreducibly
two-channel. One pure-recurrence seed also solves that cell, so the measured separation is
one of training reliability, not existence: the existence-level content of super-additivity
rests on the two single-axis walls plus the handcrafted construction for their conjunction,
and its end-to-end trained form remains open at loads beyond the pure family's effective
capacity. One prediction (channel commensurability) failed with its direction reversed; we
report the failure and partially re-establish a revised operationalization (channel
complementarity). Representational convergence is given freely by scale; capability
convergence must be purchased by access structure.
\end{abstract}

\vspace{0.5em}

\vspace{1em}

\begin{figure}[t]
  \centering
  \includegraphics[width=\linewidth]{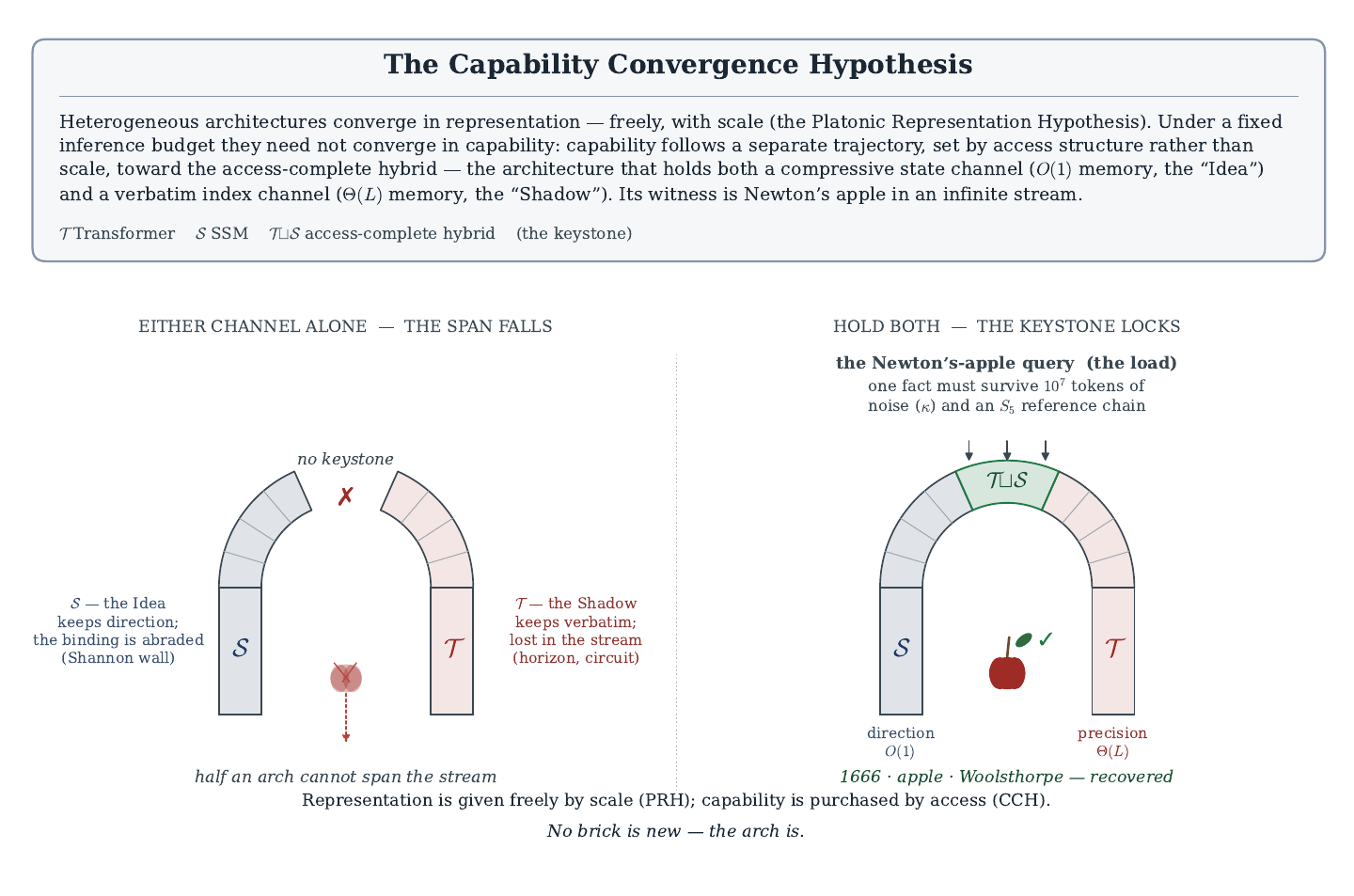}
  \caption{\textbf{Overview of the Capability Convergence Hypothesis}, in the spirit of
  PRH~\citep{huh2024platonic}. One metaphor, the arch: a compressive state channel $\Sch$
  (the ``Idea,'' $O(1)$ memory) and a verbatim index channel $\Tch$ (the ``Shadow,''
  $\Theta(L)$ memory). \emph{Left}: either channel alone is a half-arch, barred by its wall
  (Shannon, horizon, or circuit; the first two unconditional, the third conditional on
  $\mathrm{TC}^0\neq\mathrm{NC}^1$); the Newton's-apple query falls through. \emph{Right}:
  the access-complete hybrid $\Tch\sqcup\Sch$ is the keystone that locks the span, paying
  the $\Theta(\rho L)$ index the capacity floor demands under the separability assumption of
  Section~\ref{sec:cross}. Representational convergence is given freely by scale (PRH);
  capability convergence must be purchased by access structure (CCH).}
  \label{fig:overview}
\end{figure}

\section{Introduction: Two Convergences}
\label{sec:intro}

\subsection{The first convergence}
The Platonic Representation Hypothesis~\citep{huh2024platonic} opens with an elegant
observation: models trained on different data, objectives, and architectures are becoming
increasingly alike in representational geometry, agreeing more and more on which datapoints
are neighbors. PRH explains this as convergence toward a statistical projection of reality:
an ideal $Z$, the world's generative structure, that every sufficiently strong learner
approximates in its own way. Scale is the driver, diversity the fuel, convergence the
destination~\citep{kaplan2020scaling,hoffmann2022chinchilla}. Subsequent evidence has
strengthened the picture: cross-modal embedding translation with zero paired
data~\citep{jha2025vec2vec}, and the Structured State-Space Duality
result~\citep{dao2024ssd}, which places state-space recurrence and linear attention in one
matrix-mixer family (the softmax case connected by duality rather than identity).
Representational convergence appears to be a robust empirical regularity.

\subsection{A crack}
A recent result, however, drew PRH's boundary. \citet{convergence2026understanding}
found that models can align in representation (CKA $0.83$--$0.90$) yet diverge in
inferential computation, most sharply at the moment of computation: representations that
agree before a model commits to an answer (pre-decision CKA $0.875$) diverge once it does
(post-decision $0.274$), and agreement is \emph{higher} on problems the models collectively
fail ($0.897$) than on those they solve ($0.830$), a ``difficulty inversion'' shared
representation cannot explain. This crack is our point of departure: having a good
representation of $Z$ and being able to compute with it under a fixed budget are different
things. The former is PRH's territory, given freely by scale; the latter, which we call
\emph{capability}, is not free. A model may ``know'' that \emph{Newton}, \emph{apple}, and
\emph{1666} lie close together in representation space yet be unable to retrieve the
specific token ten million tokens later, continuous with the long-context utilization gap
visible in shorter settings~\citep{liu2023lostmiddle}. What is missing is not
representation but \emph{access}. A second 2026 re-examination reaches a kindred boundary
from the measurement side: at millions of retrieval candidates PRH's cross-modal alignment
itself degrades substantially~\citep{koepke2026cave}. Neither result touches the
within-modality alignment-grows-with-scale regularity (re-confirmed in
Section~\ref{sec:expA}); both say the convergence scale buys is thinner than the strong
reading of PRH assumed.

\subsection{The second convergence}
If PRH answered ``where do representations go,'' CCH attempts ``where does capability go.''
Our central claim:

\begin{hypobox}
\textbf{The Capability Convergence Hypothesis (CCH).} Under a fixed per-token inference
budget, capability does not converge with representation automatically; it converges, along a
trajectory dictated by access structure, toward a common attractor. That attractor is a
\emph{class} rather than a unique architectural endpoint: the \textbf{access-complete}
architecture, meaning any form that holds both (i) a compressive state channel whose per-token
update is $O(1)$ in the stream length $L$ and (ii) a scalable verbatim-index channel. On a
witness family, under the stated regime, this class strictly separates from both pure
families, fixed-budget compressive-state stacks (SSMs) and fixed-depth attention-only stacks:
its capability closure strictly contains the union of theirs
(Proposition~\ref{prop:nonpreserve}).
\end{hypobox}

\noindent Two levels of the claim are distinct. \emph{Architectural CCH}, this paper's
focus, concerns the model-internal case: hybrids pairing a recurrent state with global
attention. \emph{Systems CCH} applies the principle one level up: an agent stack whose tool
calls or external store supply the index channel is access-complete as a system
(Definition~\ref{def:profile}, AV2). Retrieval-augmented systems are thus an instance of
the principle, not a counterexample; what CCH rules out is capability without any scalable
index at any level. To keep the systems-level claim falsifiable, the channel definitions
are fixed ex ante: a state channel is a persistent cross-token carrier with $O(1)$
per-token update; an index channel is an addressable store growing with the span it serves,
read query-dependently; chain-of-thought is neither and is classified as leaving the regime
(trading time for space). This paper's claims and experiments are confined to Architectural
CCH; Systems CCH is a discussion-level extension (AV2), not a tested claim. The hypothesis
also has two components of different strength. The \emph{access-completeness principle}:
tasks conjoining high-entropy retrieval with serial composition require both channels at
some level; the falsifiable content lives here (P1--P3). The \emph{convergence conjecture}:
inference economics plus the principle pull the field toward access-complete designs; its
support is consistency, not demonstration (Section~\ref{sec:regularities}).

\noindent The claim mirrors PRH: many architectures converge to one representation there,
many forms to one access-completeness here; PRH's convergence is driven by scale, CCH's by
the economics of inference. One disanalogy should not be over-read: PRH's convergence is a
measured phenomenon, whereas CCH's names a field-level engineering trend plus a conjectured
attractor, not the same kind of convergence. Figure~\ref{fig:overview} previews the
argument.

\paragraph{What we do and do not claim.} We do not claim to have invented hybrid
Transformer--SSM architectures, the recall limits of SSMs, or the fixed-depth limits of
attention; each is prior work (Appendix~\ref{sec:related}), and the composite
recall-plus-state-tracking primitive our witness instantiates is due
to~\citet{reasoning2026primitives}; our addition is its capacity-and-complexity framing and
pre-registered measurement. Our contribution is the unifying hypothesis that under
fixed-budget inference capability is governed by query-time \emph{access structure} rather
than representational alignment. Concretely, we (i) formulate an access-completeness
principle through a single channel-capacity object, (ii) give a witness task separating
pure from hybrid families under an explicit regime, and (iii) connect the separation to the
industrial movement toward global-attention hybrids. We do not claim that hybrids dominate
on every task, that the result survives outside the regime (chain-of-thought, tools, and
depth-growing models are explicit escape routes; Section~\ref{sec:limitations}), or that
natural corpora reliably supply the separability the witness requires
(Assumption~\ref{ass:sep}): the witness shows access structure \emph{can} be decisive, not
that it is decisive on any particular workload.

\subsection{A thought experiment first; then its first measurements}
We state and test a hypothesis, anchored on a thought experiment, following PRH's methodology: some
propositions deserve stating clearly before they can be confirmed at scale. Full
verification would require repeated training at the $10^{10}$-parameter, $10^7$-context
scale; insisting on that first would reduce the inquiry to an engineering report and forgo
the physics-style mode in which boundaries are deduced before measured. We anchor on the
\emph{Newton's-apple problem in an infinite stream} (Section~\ref{sec:newton};
Figure~\ref{fig:concept}, parallel to Figure~1 of~\citealp{huh2024platonic}) and add what a
study can responsibly add short of frontier scale: Section~\ref{sec:experiments}
reports controlled small-scale tests of
Predictions~\ref{pred:floor}--\ref{pred:comm} plus a dissociation experiment, under frozen
criteria, with every miss (including one outright failure) reported.

\begin{figure}[t]
  \centering
  \includegraphics[width=0.62\linewidth]{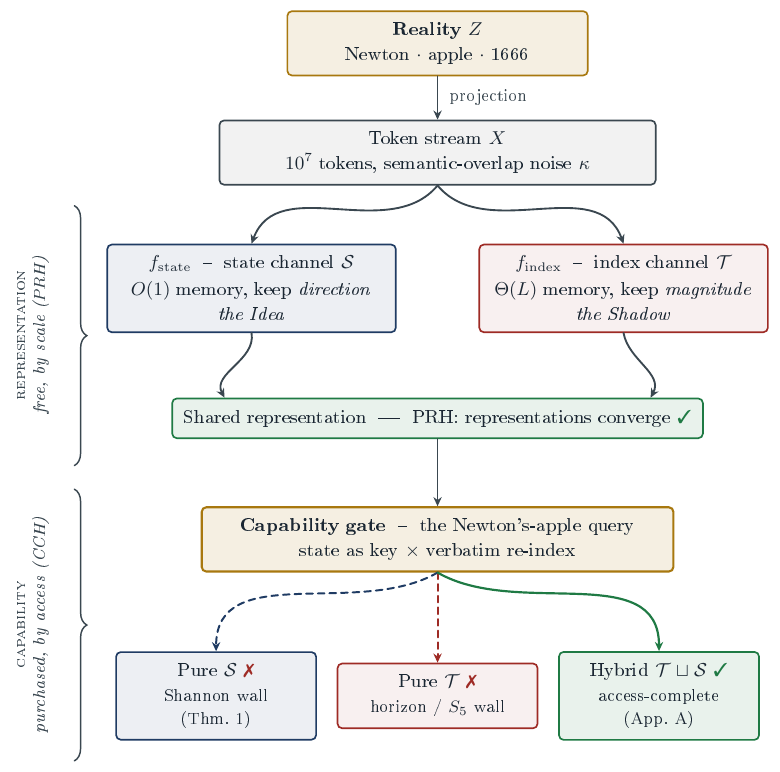}
  \caption{The CCH concept diagram, mirroring Figure~1 of~\citet{huh2024platonic}. Reality
  $Z$ projects into a stream $X$; two channels read it, a state channel $f_{\mathrm{state}}$
  (the ``Idea,'' $O(1)$ memory) and an index channel $f_{\mathrm{index}}$ (the ``Shadow,''
  $\Theta(L)$). At the representational level the two converge (PRH); at the capability level
  only the access-complete hybrid passes the Newton's-apple gate (conditional on
  Assumption~\ref{ass:sep}; the horizon deferred, not eliminated).}
  \label{fig:concept}
\end{figure}

\paragraph{Why our evidence takes a different form from PRH's.} PRH's claim is
observational (representations are converging), so measuring released checkpoints suffices,
and its lineage has stayed in that evidential
mode~\citep{jha2025vec2vec,convergence2026understanding,koepke2026cave}. CCH's central
claim is modal and causal: a pure family \emph{cannot} solve the witness at any scale under
$\Rgm$, and the composition can. No measurement of released models can discriminate such a
claim, because released models confound architecture with data, recipe, and scale (the
identification limit flagged for R1--R4); a causal claim needs lower bounds and controlled
intervention. The portfolio therefore has three layers, each doing work the previous
cannot: observational (the 140-model census; the public-checkpoint measurements of
Section~\ref{sec:expA}, using PRH's own metric), formal (the walls of
Section~\ref{sec:newton}), and interventional (the pre-registered trained witnesses of
Section~\ref{sec:experiments}).

\paragraph{Claims and their evidential status.} Table~\ref{tab:claims} states each claim
layer once, with status and evidence; every scope qualifier in the paper instantiates one
of its rows. The falsifiable core is the access-completeness principle (rows 1--7); the
convergence component (row 9) is a motivating conjecture, and we ask that the paper be
judged on the former.

\begin{table}[t]
\centering
\small
\caption{The paper's claims, stratified by evidential status. ``Consistent-with'' means the
evidence does not discriminate the claim from rival explanations (notably pure serving
economics).}
\label{tab:claims}
\begin{tabular}{@{}p{0.44\linewidth}p{0.22\linewidth}p{0.26\linewidth}@{}}
\toprule
\textbf{Claim} & \textbf{Status} & \textbf{Evidence} \\
\midrule
1. Capacity floor for high-entropy independent bindings & formal, unconditional & Theorem~\ref{thm:shannon} \\
2. Truncation of any fixed window & formal, unconditional & Prop.~\ref{prop:boundaries} \\
3. Fixed-depth attention-only composition limit & formal, conditional on $\mathrm{TC}^0{\neq}\mathrm{NC}^1$ & Prop.~\ref{prop:circuit} \\
4. Hybrid solves the witness & existence construction, conditional on Assumption~\ref{ass:sep} & Appendix~\ref{app:construction} \\
5. Single-axis separations learnable by gradient descent & measured, small scale, pre-registered & Exp.~B, C \\
6. End-to-end conjunction separation & training-reliability only; existence-level open & Exp.~E \\
7. Channel complementarity (P3$'$) & partial (joint necessity, localization) & Exp.~D-ii \\
8. Alignment does not predict capability & observational dissociation & Exp.~A \\
9. Industry converges to the access-complete class & conjecture; census is consistent-with only & R1--R4, P4--P5 \\
10. Natural corpora satisfy write-time separability & scoped at surface-form level (mixed); embedding-level and incidence open & App.~\ref{sec:collnat}, Limitation~(iii) \\
\bottomrule
\end{tabular}
\end{table}

\paragraph{Structure.} The paper follows PRH's arc: \emph{why} representation does not buy
capability (Section~\ref{sec:operators}: one channel-capacity axis under Shannon's
bound~\citep{shannon1948mathematical,cover2006elements}, the witness, the three walls);
\emph{what} capability converges to (Section~\ref{sec:closure}); the evidence
(Sections~\ref{sec:regularities}--\ref{sec:experiments}); then a discussion of alternative
explanations (Section~\ref{sec:alternatives}) and predictions, limitations, and implications
for measurement and reporting (Section~\ref{sec:predictions}). Demarcation, proofs, the construction, the census, and
full experimental details are in the appendices.

\section{Why Representation Does Not Buy Capability: One Channel, Three Walls}
\label{sec:operators}

This section develops a single thesis: Transformer and SSM are best compared as two
history-to-query channels differing in fidelity, attention approximating high-fidelity
indexed access and an SSM realizing structured, compressed mixing. The comparison builds on
the matrix-mixer unification of~\citet{dao2024ssd}, used as a shared coordinate rather than
to erase the dynamical differences between the families, and it interfaces directly with
Shannon's channel theory.

\subsection{Setup and notation}
A causal sequence model on a token stream $X_{1:t}$ induces at each query time $t$ an
\emph{accessible state} $\Sigma_t$ from which $\hat{y}_t$ is computed: the recurrent hidden
state $h_t$ for an SSM~\citep{gu2023mamba}, the servable KV cache for a
Transformer~\citep{vaswani2017attention,kwon2023pagedattention}, i.e.\ everything the model
may consult at query time. Write $B := \Hch(\Sigma_t)$ for its \emph{capacity} in bits.

\subsection{The state-channel postulate}
CCH rests on one observation, a direct application of the data-processing
inequality~\citep{cover2006elements}.

\begin{postulate}[State-channel bound]
\label{post:channel}
Fix parameters $\theta$ and query token $q_t$. The accessible state $\Sigma_t$ is the
unique data-bearing channel from history to prediction: conditioned on $(q_t,\theta)$ the
readout reads $X_{1:t}$ only through $\Sigma_t$, so
$X_{1:t} \rightarrow \Sigma_t \rightarrow \hat{y}_t$, and by data processing
\begin{equation}
I(X_{1:t};\,\hat{y}_t \mid q_t,\theta) \;\le\; I(X_{1:t};\,\Sigma_t \mid q_t,\theta)
\;\le\; \Hch(\Sigma_t) \;\le\; B.
\label{eq:dpi}
\end{equation}
($\theta$ carries no information about $X_{1:t}$; $q_t$ is explicit because the readout,
including any top-$k$ selection, may depend on it, Lemma~\ref{lem:topk}.) Representation
quality may be unbounded; the bits through this channel cannot exceed $B$.
\end{postulate}

This inequality is the watershed between PRH and CCH: PRH concerns representation
\emph{quality}, unbounded with scale~\citep{huh2024platonic}; CCH concerns channel
\emph{capacity}, fixed by access structure. A perfect representation of the Newton--apple
binding that no longer fits in $\Sigma_t$ after ten million tokens never reaches the
output. Figure~\ref{fig:shannon} places both operators on this channel.

\begin{figure}[t]
  \centering
  \includegraphics[width=0.86\linewidth]{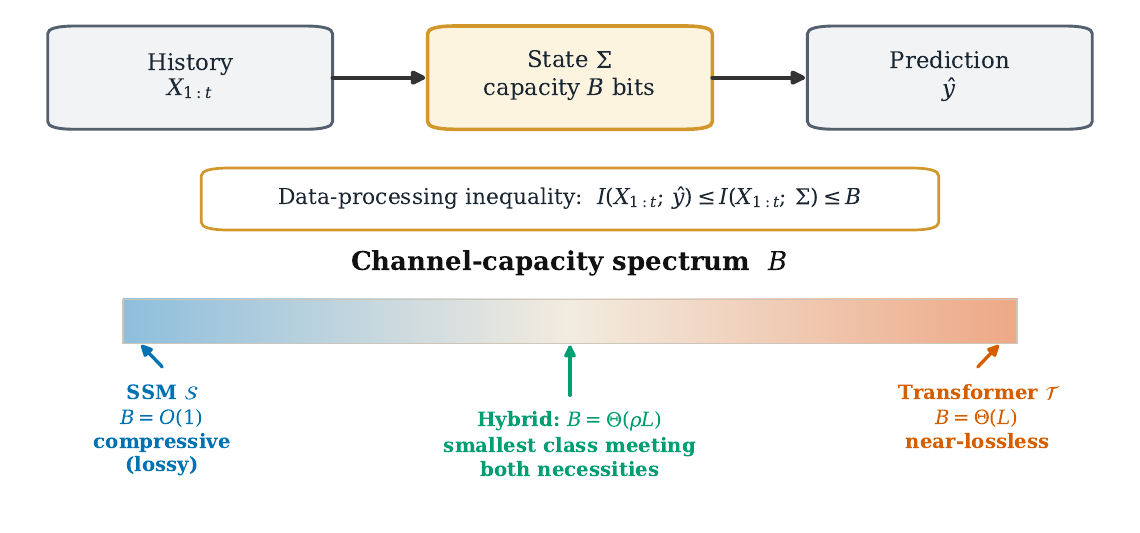}
  \caption{Transformer and SSM as two ends of one channel-capacity spectrum
  (Postulate~\ref{post:channel}): the SSM at the compressive end ($B=O(1)$, lossy), the
  Transformer at the near-lossless end ($B=\Theta(L)$), and the hybrid operating point
  $B=\Theta(\rho L)$ meeting both task necessities at a reduced servable-memory coefficient
  ($\rho\ll1$), not a smaller growth class (Proposition~\ref{prop:minimal}).}
  \label{fig:shannon}
\end{figure}

\subsection{Making $B$ precise: finite-precision regularization}
For a Transformer, $B = \Theta(L)$: the servable cache grows linearly. For an SSM,
$\Sigma_t$ is continuous and ``capacity in bits'' requires care; we regularize via finite
precision, matching both theory and hardware.

\begin{definition}[Effective state capacity]
\label{def:capacity}
For $\Sigma_t \in \mathbb{R}^{d_{\mathrm{state}}}$ at $p$ bits per coordinate (idealized
$p = O(\log L)$; in practice bf16/fp32), the capacity is the upper bound
$B := d_{\mathrm{state}}\,p \ge \Hch(\Sigma_t)$ (an inequality: reachable states need not
fill the code space; equivalently, the metric entropy of the reachable set at resolution
$2^{-p}$~\citep{cover2006elements}).
\end{definition}

\begin{remark}
Were $\Sigma_t$ the full history (the Transformer limit, $B=\Theta(L)$),~\eqref{eq:dpi}
would be a tautology; non-triviality comes from the resource constraint, which is why CCH is
a statement about budgeted inference, not expressivity in the abstract.
\end{remark}

\begin{definition}[Inference resource profile]
\label{def:profile}
``Fixed per-token inference budget'' holds an entire \emph{resource profile} fixed, not a
scalar, because asymptotic order is too coarse to separate the architectures of interest:
\[
\mathcal{P} \;=\; \big(\,c_{\mathrm{upd}},\; \mu,\; \beta_q,\; p,\; D,\; \tau_{\mathrm{cot}}\,\big),
\]
per-token update compute; the \emph{servable-memory coefficient} $\mu$ (query-time bytes
per token, the constant in front of the memory-growth class); read bandwidth $\beta_q$;
precision $p$; depth $D$ (fixed in $L$); chain-of-thought budget $\tau_{\mathrm{cot}}$
(zero in $\Rgm$). Two architectures in the same memory-growth class can differ by an order
of magnitude in $\mu$, $\beta_q$, and serving cost (for constant $\rho$ the hybrid's
$O(\rho L)$ and full attention's $O(L)$ are the same asymptotic class), so CCH's economic
claims live in these coefficient/bandwidth coordinates, not big-$O$ order. The regime pins
memory to a growth class with fixed coefficient, update compute to $O(1)$, and depth to a
constant, while an index channel's query-time \emph{read} necessarily scales as
$\Theta(\rho L)$ (demanding $O(1)$ read would exclude every access-complete system by
fiat). The accurate name for $\Rgm$ is thus the \emph{fixed-depth, no-CoT,
fixed-memory-coefficient regime}.
\end{definition}

\subsection{The two operators located on the spectrum}
\paragraph{Transformer $\Tch$: near-lossless attention.}
Self-attention holds the entire history's keys and values at query time,
$Z^{\Tch} = \mathrm{Softmax}(QK^\top/\sqrt{d})\,V$: a near-lossless channel in the
addressability sense (per-position addressable state, not literal losslessness, since keys
and values are projected and quantized), with servable capacity $B_{\Tch} = \Theta(L)$. Its
capability signature is exact retrieval~\citep{arora2023zoology}; its price is the
$\Theta(L)$ memory growth that becomes the serving
wall~\citep{dao2022flashattention,kwon2023pagedattention}.

\paragraph{SSM $\Sch$: compressive attention.}
A state-space model compresses history into a fixed-size state, $h_t = A_t h_{t-1} + B_t
x_t$, $y_t = C_t h_t$~\citep{gu2023mamba}, the selective successor to structured state
spaces~\citep{gu2022s4} and linear attention~\citep{katharopoulos2020linear}; Structured
State-Space Duality~\citep{dao2024ssd} places it in the same matrix-mixer family as
attention, with the mixing matrix constrained to low-rank structure (gating and
state-update dynamics remain distinct). $\Sch$ is a lossy-compressive channel with fixed
accessible capacity $B_{\Sch} = \Theta(d_{\mathrm{state}}) = O(1)$ in $L$. Its signature,
in heuristic terms, is ``keep the direction, drop the magnitude'': a coarse summary of
global structure survives while the detail binding a specific symbol to a specific value
does not. The shorthand is informal; the rigorous content is $B_{\Sch}=O(1)$, whose
intuitive face is the empirically measured recall
gap~\citep{jelassi2024repeat,arora2023zoology,jelassi2025recency}.

\subsection{The unified statement: capability lives at the union}
CCH's central intuition is that real long-range capability needs \emph{both ends at once}:
frugal maintenance of long-range direction (the $\Sch$ end) \emph{and} exact retrieval of a
key binding (the $\Tch$ end). No single operator sits at the union. We make the boundary
behavior precise.

\begin{proposition}[Convergence boundaries under infinite corpus]
\label{prop:boundaries}
Under an unbounded corpus and the regime of Definition~\ref{def:profile}:
\textup{(i)}~a near-lossless Transformer must truncate to a window $L_{\max}$, a cliff
(near-perfect inside, blind outside); \textup{(ii)}~a lossy-compressive SSM must compress
ever more into a fixed state, a slope (asymptotically faithful in global direction,
corrupted in local binding).
\end{proposition}

The cliff and the slope are distinct failure geometries, but both are failures. Retrieving
one precise fact after $10^7$ tokens demands seeing beyond the window \emph{and} telling
bindings apart, which is why we locate the fixed point of capability at the composition of
the two.

\subsection{The witness: Newton's apple in an infinite stream}
\label{sec:newton}

PRH uses Plato's cave to picture representations converging toward the ideal. CCH anchors on
a capability scenario, one precise binding planted in an unbounded stream and queried $10^7$
tokens later, cast as a one-way communication game so that the \textsc{Index} lower
bound~\citep{kremer1999randomized} and Fano's inequality~\citep{cover2006elements} apply
directly.

\begin{definition}[Protocol $\mathrm{NA}(N,b,B;\kappa)$]
\label{def:na}
Nature plants $N$ key--value bindings $\{(r_i,v_i)\}_{i=1}^N$ into a stream, each carrying $b$
bits (the target is one such binding, e.g.\ \emph{``1666 $\cdot$ apple $\cdot$ Woolsthorpe''}).
It interleaves $\Theta(L)$ distractor tokens drawn with a semantic-overlap kernel $\kappa$ (the
worst case: distractors topically aligned with the target, maximizing interference with any
compressed state), and drives the referent's group action $r_t$ by an $S_5$ instruction stream.
At the end it issues a query for the target value. The model carries budget $B$ (its query-time
accessible state, per Definition~\ref{def:capacity}). By Yao's minimax
principle~\citep{yao1977probabilistic}, a lower bound against the hardest input distribution
implies the same bound for the best deterministic model on that distribution.
\end{definition}

\noindent Throughout, the ``infinite stream'' is a task-family limit: each length-$L$
instance is served by resources permitted to grow with that instance ($\Theta(\rho L)$ index
for an access-complete system). It is not a claim about processing an unbounded stream on
fixed absolute hardware, where every linear-index system, the hybrid included, eventually
truncates (Section~\ref{sec:cross}).

The scenario also fixes the philosophical stakes. An old intuition, formalized in the
compression view of
intelligence~\citep{mahoney1999text,legg2007universal,deletang2024compression}, holds that
intelligence is not storing more of the stream but finding its order; the walls sharpen
this into an impossibility statement, because a maximally entropic binding has no order to
find. Whatever
structure a compressive state distills, the incompressible residue is lost
(Theorem~\ref{thm:shannon}) unless it can be re-read verbatim. Capability at long context is
therefore a conjunction, compress what has structure and index what does not, and
access-completeness is that conjunction made architectural.

\subsection{Wall I: the Shannon wall (unconditional)}
Consider any finite-state compressive architecture (a pure SSM), with accessible capacity $B$
bits. We bound exact recall.

\begin{theorem}[Shannon wall]
\label{thm:shannon}
Let the query target be a uniformly chosen binding $V_I$ with $I \sim \mathrm{Unif}[N]$, each
binding $b$ bits and the $N$ bindings mutually independent. Let $\Sigma$ be the model's
\emph{pre-query persistent} accessible state, formed before the index $I$ is revealed; the
joint law is that the bindings are drawn, $\Sigma$ is formed from the stream, and then
$I \sim \mathrm{Unif}[N]$ is drawn independently of $(V_{1:N},\Sigma)$. Assume
$\Hch(\Sigma)\le B$; any external store the system
can read at query time counts in full toward $\Hch(\Sigma)$, and the answer is a function of
$(\Sigma, I)$ (query-dependent reads from $\Sigma$ add nothing, Lemma~\ref{lem:topk}). Then
the per-element accessible information obeys
\begin{equation}
I(V_I;\,\Sigma,I) \;\le\; \frac{B}{N},
\label{eq:perelem}
\end{equation}
and consequently the conditional entropy of the target is bounded below by
\begin{equation}
\Hch(V_I \mid \Sigma, I) \;\ge\; b\,\max\Big\{0,\;1 - \tfrac{1}{x}\Big\}, \qquad x := \frac{Nb}{B},
\label{eq:floor}
\end{equation}
Two operational floors follow, by distinct routes:
\begin{itemize}[leftmargin=1.4em,topsep=2pt,itemsep=1pt]
\item\emph{Log-loss floor (no Fano).} The Bayes-optimal expected log-loss given $(\Sigma,I)$
\emph{equals} $\Hch(V_I\mid\Sigma,I)$, so by Gibbs' inequality (cross-entropy $\ge$ conditional
entropy) every predictor obeys
$\mathbb{E}[-\log p(V_I\mid\Sigma,I)] \ge \Hch(V_I\mid\Sigma,I) \ge b(1-1/x)$.
\item\emph{Error-probability floor (Fano).} For exact recall, Fano's
inequality~\citep{cover2006elements} gives
$P_{\mathrm{err}} \ge \big(\Hch(V_I\mid\Sigma,I)-1\big)/b \ge \max\{0,\,(1-1/x) - 1/b\}$
(both floors are trivial below $x=1$, as they must be).
\end{itemize}
\end{theorem}

Once the load ratio $x = Nb/B$ exceeds $1$, no finite-state model can avoid losing
$b(1-1/x)$ bits of the target, however good its representation. The bound holds for
independent bindings at the stated entropy; structured or compressible bindings evade the
premise, not the theorem (the compressibility axis of Appendix~\ref{sec:expB} measures
this). Only a blurred direction (``this is a Newton story'') survives, matching the
measured copying and recall gaps of
SSMs~\citep{jelassi2024repeat,arora2023zoology,jelassi2025recency}. The wall is
unconditional: information theory only. Figure~\ref{fig:fano} plots the
floor~\eqref{eq:floor}.

\begin{figure}[tp]
  \centering
  \includegraphics[width=\linewidth]{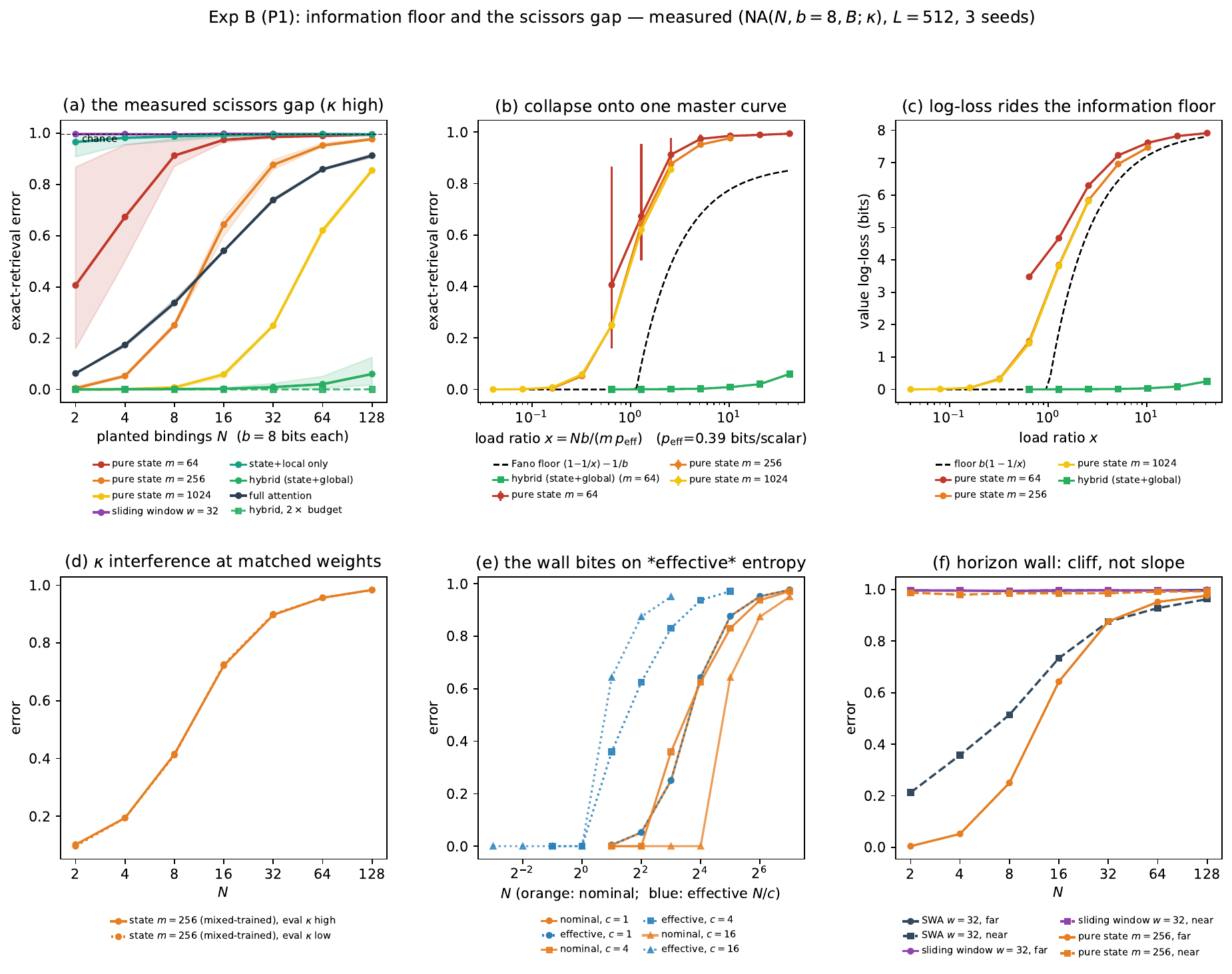}
  \caption{The information floor and the scissors gap, \emph{measured}
  (Experiment~B, Appendix~\ref{sec:expB}: $\mathrm{NA}(N,b{=}8,B;\kappa)$,
  $L{=}512$, 3 seeds; bands span seeds). \textbf{(a)} Exact-retrieval
  error vs.\ load $N$: every pure-state curve ($m\in\{64,256,1024\}$ scalars)
  rises to the chance ceiling while the hybrid (the \emph{same} $m{=}64$ state
  plus one global-attention layer) stays near zero, the scissors of
  Prediction~\ref{pred:floor}; dashed green = doubled budget (24k, six seeds),
  mean error $\le 0.001$ at every $N$, so the $N{=}128$ near-miss at 12k was
  undertraining. A window-only index (teal) does not rescue the state; 2-layer
  full attention (navy) is not an upper bound at this scale.
  \textbf{(b)} The three pure-state curves against the load ratio
  $x=Nb/(m\,p_{\mathrm{eff}})$ with the single frozen
  constant $p_{\mathrm{eff}}{=}0.39$ bits/scalar: 50\%-crossings collapse to
  $x\in[0.82,1.03]$ ($m{=}256$ vs $m{=}1024$ within $0.024$ error; $m{=}64$,
  at the edge of trainability, deviates by up to $0.15$). \textbf{(c)} Value
  log-loss tracks the shape of the $b(1-1/x)$ floor (dashed) under the fitted
  effective capacity, a one-parameter empirical collapse consistent with
  Theorem~\ref{thm:shannon} rather than a measurement of $\Hch(\Sigma)$ ($p_{\mathrm{eff}}$
  is far below hardware precision, so the \emph{nominal} floor is loose here);
  the hybrid's stays near zero because its $\Omega(Nb)$ load rides the index
  channel (Assumption~\ref{ass:sep}). \textbf{(d)} No detectable
  semantic-overlap ($\kappa$) effect at matched weights, a pre-registered
  prediction that returned a null, reported as such. \textbf{(e)} With
  compressible values the wall tracks effective entropy at $c{=}16$ but not at
  $c{=}4$: mixed evidence. \textbf{(f)} The horizon wall is a cliff: the
  sliding-window model is at chance for out-of-window bindings at every $N$.
  \emph{Measured data; protocol and per-seed results in
  \texttt{experiments/}.}}
  \label{fig:fano}
\end{figure}

\subsection{Wall II: the horizon wall (unconditional)}
For a budgeted pure attention architecture with window $L_{\max} \ll N$, the opening
``apple'' binding has slid out of the KV window and is entirely lost: servable capacity is
$B_{\Tch} = \Theta(L_{\max}) = o(Nb)$, so Theorem~\ref{thm:shannon} applies to the
Transformer as well. Top-$k$ sparse reading does not raise the bound:

\begin{lemma}[Top-$k$ reading does not raise the bound]
\label{lem:topk}
Let the selected key set $S$ and its values $V_S$ be a deterministic function of $(\Sigma,q)$
for query $q$. Then conditioning on $q$, the Markov chain $X_{1:t} \rightarrow \Sigma
\rightarrow (S,V_S)$ holds, and by the data-processing inequality $I(X_{1:t};\,S,V_S \mid q)
\le I(X_{1:t};\,\Sigma) \le B$. The information ``leaked'' by \emph{which} keys are selected is
already contained in $\Sigma$.
\end{lemma}

The two failures share one root: a query-time state of size $o(Nb)$. The horizon wall names
a resource deficit (index reach), not an attention-specific defect: any system with an
absolutely fixed window hits it, and any system that escapes it, the hybrid included, does
so by letting servable memory grow linearly at some coefficient (Section~\ref{sec:cross}).

\subsection{Wall III: the circuit wall (conditional on $\mathrm{TC}^0 \neq \mathrm{NC}^1$)}
Now idealize: the entire $10^7$ tokens inside the window. Retrieval is no longer the
problem; the \emph{composition of reference} is. What ``it'' refers to depends on a chain
of composed state updates, in the worst case the word problem over the non-solvable group
$S_5$, $\mathrm{NC}^1$-complete by Barrington's theorem~\citep{barrington1989bounded}. A
fixed-depth, log-precision, softmax-attention stack (in the model of the cited bound) lies
in $\mathrm{TC}^0$~\citep{merrill2023parallelism}; chain-of-thought is exactly what escapes
this class for serial problems~\citep{li2024chain}. The wall is not specific to attention:
an ordinary diagonal SSM likewise lies in $\mathrm{TC}^0$ and equally fails to
length-generalize $S_5$~\citep{merrill2024illusion,sarrof2024expressive}, so this is one
wall barring both, not two complementary ones. The only known escape within $\Rgm$ is a
state-tracking-augmented linear recurrence (negative eigenvalues, Householder
products~\citep{grazzi2025deltaproduct,siems2025deltaproduct}), precisely the state channel
the hybrid supplies, not a property of SSMs in general.

\begin{proposition}[Circuit wall]
\label{prop:circuit}
Under the regime $\Rgm$ = (fixed depth not growing with $L$, precision $p = O(\log L)$,
softmax attention as modeled by the cited $\mathrm{TC}^0$ upper bound, no chain-of-thought),
neither a pure attention family nor an
ordinary diagonal SSM can length-generalize the composition of $S_5$ group actions, conditional
on $\mathrm{TC}^0 \neq \mathrm{NC}^1$.\footnote{The upper bound we invoke is that such an attention
stack is simulable in (uniform) $\mathrm{TC}^0$~\citep{merrill2023parallelism}; the qualifiers
on depth, precision, and the attention model are exactly those of that result, and the barrier
is for length generalization of the family, not for any fixed length.} Here $\Theta(L)$ KV
memory does not help: the bottleneck is sequential composition within fixed depth, not memory
capacity. Correspondingly, escaping the wall requires not more memory but a more expressive
state recurrence (negative eigenvalues and Householder
products~\citep{grazzi2025deltaproduct,siems2025deltaproduct}).
\end{proposition}

\paragraph{The three walls are not co-equal.} They name three resource deficits (index
capacity, index reach, serial composition), and ``crossing'' throughout means supplying the
missing resource at its stated price, never evading a bound. The load-bearing separation is
carried by the two unconditional walls: in $\Rgm$, information theory alone bars both pure
families. The circuit wall is conditional and supplementary, closing off an idealized
Transformer granted unbounded memory, a configuration already outside $\Rgm$; even if
$\mathrm{TC}^0 = \mathrm{NC}^1$, Walls~I--II still separate the pure families under budget.
Chain-of-thought and tools are orthogonal escape axes that leave $\Rgm$
entirely~\citep{li2024chain} (Section~\ref{sec:limitations}, AV2).

\subsection{How the hybrid crosses, under a separability assumption}
\label{sec:cross}
Now consider an access-complete hybrid holding both
channels~\citep{lieber2024jamba,de2024griffin,blakeman2025nemotronh}. We argue it solves
Newton's apple not through greater scale but through the composition of two old mechanisms.
The contraction step relies on a separability assumption, without which worst-case overlap
defeats any state channel; the contraction is therefore a conditional claim.

\begin{assumption}[Write-time code separability]
\label{ass:sep}
The $N$ target bindings carry \emph{anchor codewords} drawn from a minimum-distance code,
recoverable from the query at query time, and the $\Theta(L)$ semantic-overlap distractors
stay at distance $\ge \delta_\kappa > 0$ from every codeword in the index's matching space.
Equivalently: distractors may be topically near the target (high $\kappa$ in raw semantics)
yet remain code-separated at write time, so that distinct valid addresses, and every
distractor, keep a minimum margin in the space the index channel matches against. The
assumption constrains the \emph{index} channel's geometry only; it asks nothing of the state
channel about the (unknown) future target.
\end{assumption}

\noindent Assumption~\ref{ass:sep} is the load-bearing premise: the protocol plants
distractors maximizing \emph{semantic} interference, and separability is purchased in a
\emph{coded} channel, not raw semantics (A.4). It enters at exactly one step (point~3);
elsewhere the construction is unconditional. One information-flow point: the query key
arrives at the stream's end, so the state cannot know which binding is the target and we
claim no such thing; it maintains online only the composed referent $r_t$ (the key code
arrives with the query). The division of labour is that the state computes the address and
the index resolves it, not that the state preserves the target and shrinks a candidate set
(a narrative from earlier drafts, retracted). Its architecture:

\begin{enumerate}[leftmargin=1.4em,itemsep=2pt]
\item \textbf{$\Sch$ composes the address.} Through the noise region the $\Sch$ channel
maintains, at $O(1)$ update (width scales with $\log N$ and $b$, not $L$; A.6), the online
quantity: the referent $r_t \in S_5$, a product of $n_h \ge 4$ Householder reflections (any
$S_5$ permutation is $\le 4$ transpositions), realizable by
DeltaProduct~\citep{siems2025deltaproduct} with eigenvalues in $[-1,1]$, the range that
unlocks state tracking~\citep{grazzi2025deltaproduct}. It preserves no binding's lexeme; it
cannot and need not.
\item \textbf{The write-time code keeps bindings resolvable (Assumption~\ref{ass:sep}).}
Bindings are written under anchor codewords code-separated from all distractors, so a query
formed from the composed address resolves to the right binding; without the assumption the
hybrid inherits the pure-family floor, which we test (the collision control pins it to the
floor at zero code distance, Appendix~\ref{sec:expB}). We claim no candidate-set
contraction and no sublinear read: the index stores $\Theta(\rho L)$ patterns scanned in
full, and Proposition~\ref{prop:minimal}'s economics rest on $\rho$, not sublinear access.
\item \textbf{$\Tch$ retrieves exactly.} The $\Tch$ channel (or a Hopfield
read~\citep{ramsauer2021hopfield}) matches the composed address against the code-separated
index; per-query error is controlled by the code margin (A.4).
\end{enumerate}

No single step suffices: without $\Sch$'s direction $\Tch$ hits the Shannon or horizon wall
on $10^7$ tokens; without $\Tch$'s precision $\Sch$ slides down the magnitude-neglect slope.
Composed, the apple is recovered.

\paragraph{What ``crossing'' does and does not mean: two claims, not one.} The walls are
priced, not evaded. Claim~(i), absolute capacity: any architecture with $o(Nb)$ query-time
state fails exact retrieval (Theorem~\ref{thm:shannon}, unconditional); the hybrid
satisfies the bound the only way anything can, by paying $\Omega(Nb)$ index capacity
($\Theta(\rho L)$ KV). Claim~(ii), division of labour at paid capacity: given the linear
index, a compressive recurrence supplies the serial composition a fixed-depth attention
stack cannot (conditionally), at $O(1)$ update and a reduced coefficient. Neither ``beats''
a wall: under an absolute budget the hybrid's horizon is deferred by $1/\rho$, not
eliminated ($\Theta(\rho L) = \Theta(L)$). So hybrid-versus-full-attention is a coefficient
statement, never a class separation; the only class-level separations are against
$o(Nb)$-state architectures (capacity) and fixed-depth no-recurrence stacks (composition,
conditional). The symmetry runs both ways: a full-attention stack at the same $\mu$ is
access-complete on axis~(i) and lacks only axis~(ii); a pure recurrence lacks axis~(i). The
separation is the conjunction. Its hardness is combinatorial: the $S_5$ sub-task alone is
solved by a pure DeltaProduct, exact retrieval alone by a pure Transformer in-window; what
no pure family solves under $\Rgm$ is the conjunction. No brick is new
(Barrington~\citealp{barrington1989bounded}; \textsc{Index}~\citealp{kremer1999randomized};
Hopfield~\citealp{ramsauer2021hopfield}; DeltaProduct~\citealp{siems2025deltaproduct}); the
assembly is.

\paragraph{Epistemic status.} The construction is an existence result under coded
separability: weights and a code-separated witness family on which the hybrid crosses all
three walls while no pure family does at any scale. It does not claim that natural language
supplies such anchors or that gradient descent finds the construction; these open questions
(Limitations~(iii)--(iv)) are what the pre-registered tests begin to probe. The witness
proves access structure \emph{can} be decisive, not that it \emph{is} on any natural
corpus.

\section{What Capability Converges To: The Access-Complete Class}
\label{sec:closure}

This section states Newton's apple as CCH's mathematical core: a single witness separation,
a task family no pure Transformer and no pure SSM solves at any scale while their
composition does (Proposition~\ref{prop:nonpreserve}).

\paragraph{The capability closure.} Write $\Cinf(\mathfrak{a})$ for the set of task
families an architecture family $\mathfrak{a}$ eventually solves under scaling within
$\Rgm$ (accuracy $\ge 2/3$, with length generalization). ``Closure'' is not a metaphor:
solvability induces a Galois connection whose composite is a closure operator in the
order-theoretic sense (setup, axioms, and a caution against over-reading the algebra in
Appendix~\ref{app:galois}); the witness argument needs only this definition.

\subsection{The core proposition: super-additivity under composition}

\paragraph{The classes, defined by resources rather than trademarks.} $\Sch$ is the
\emph{compressive class}: architectures whose total query-time state is $o(Nb)$ on the
witness. LSTM, DeltaNet, DeltaProduct, and RWKV-style recurrences all belong, including
those that clear the circuit wall, because expressive recurrence buys composition, not
capacity (measured: the LSTM passes every state-tracking cell of Appendix~\ref{sec:expC}
and would still be barred at load). $\Tch$ is the \emph{indexed class}: fixed-depth
attention-only stacks with a scalable $\Theta(L)$ index and no expressive recurrence,
barred on composition (conditionally) rather than capacity. Defining classes by resource
bounds fixes what ``at any scale'' means: scaling within $\Rgm$ may grow parameters, width,
and data while depth stays under a uniform $D_0$ (as the $\mathrm{TC}^0$ bound requires)
and $\Sch$ keeps $o(Nb)$ state; a recurrence grown to $\Omega(Nb)$ state has not refuted
the proposition but left $\Sch$, paying the capacity price.

The three walls of Newton's apple, translated into the language of capability closure, are:

\begin{proposition}[Capability closure does not preserve composition]
\label{prop:nonpreserve}
Let $\Tch$ and $\Sch$ be the indexed and compressive classes defined above. Let
$\Tch\sqcup\Sch := \langle \Tch,\Sch\rangle_{\mathrm{comp}}$ be the \emph{compositional join}:
the smallest architecture family containing all layer-interleavings of $\Tch$ and $\Sch$
blocks, including the degenerate interleavings with zero blocks of either type, so that
$\Tch \subseteq \Tch\sqcup\Sch$ and $\Sch \subseteq \Tch\sqcup\Sch$ and the right-hand union
below is automatically contained in the left-hand closure. Then under regime $\Rgm$,
Assumption~\ref{ass:sep}, and the conjecture $\mathrm{TC}^0 \neq \mathrm{NC}^1$,
\begin{equation}
\Cinf(\Tch \sqcup \Sch) \;\supsetneq\; \Cinf(\Tch) \;\cup\; \Cinf(\Sch),
\label{eq:nonpreserve}
\end{equation}
where the right-hand $\cup$ is ordinary union of task closures. Newton's apple is a
\emph{witness} of the strict inclusion, and the two hypotheses enter at exactly one point
each: Assumption~\ref{ass:sep} places the witness in the left-hand side
(Appendix~\ref{app:construction}), and $\mathrm{TC}^0 \neq \mathrm{NC}^1$ excludes it from
$\Cinf(\Tch)$; its exclusion from $\Cinf(\Sch)$ is unconditional
(Theorem~\ref{thm:shannon}).
\end{proposition}

This is the precise meaning of ``capability beyond emergence'': not that the hybrid is
better on every task (trivial and often false), but that there \emph{exists} a task family
no pure architecture solves at any scale while the hybrid does. New capability comes from
structure, not scale; we call this \emph{compositional emergence}.

\begin{corollary}[Super-additivity under composition]
\label{cor:nonjoin}
Under the conditions of Proposition~\ref{prop:nonpreserve}, capability is strictly
super-additive under architectural composition on the witness set (in the two-lattice sense
of Remark~\ref{rem:lattices}; the plain content is joint necessity of two resource
primitives). The gap $\Cinf(\Tch\sqcup\Sch) \setminus (\Cinf(\Tch)\cup\Cinf(\Sch))$
contains the tasks requiring both an $\Omega(Nb)$-indexable channel and an $O(1)$-composable
state; the witness shows it nonempty (not that it contains only such tasks).
\end{corollary}

\noindent A caution (Remark~\ref{rem:lattices}): the inclusion~\eqref{eq:nonpreserve} mixes
two lattices, so it is not the failure of a union-preserving homomorphism; the plain
content is that the composite family's reachable tasks strictly exceed the union of the
components'.

\subsection{The convergence conjecture (motivating, not demonstrated)}
So far we have argued that the hybrid is stronger; CCH's further claim, that capability is
\emph{converging} to the hybrid, is an economics-motivated conjecture (row~9 of
Table~\ref{tab:claims}): nothing in this section demonstrates it, and the census evidence
of Section~\ref{sec:regularities} is consistency-level only. It rests on two pillars.

\paragraph{Pillar 1 (multiple realizability).} Access-completeness is not tied to one
implementation. Layer-wise Mamba/attention
hybrids~\citep{lieber2024jamba,blakeman2025nemotronh}, lightning-attention hybrids with
periodic global softmax~\citep{minimax2025}, gated-DeltaNet
hybrids~\citep{qwen3next2025}, and gated-linear-attention hybrids with global
attention~\citep{kimi2025linear} all instantiate the same access-complete pattern (that
they share one capability closure is the hypothesis, not an established fact); forms whose
only attention is a fixed local window~\citep{de2024griffin,arora2024based} hold the
compressive channel but merely a local index and are not access-complete
(Appendix~\ref{app:census}). The divergence is in the operator, the convergence in the
structure: PRH's ``many architectures, one representation'' mirrored on the capability
axis.

\paragraph{Pillar 2 (the driver).} PRH's convergence is driven by scale, CCH's by the
economics of inference: a $\Theta(L)$ KV-cache is physically unservable in the agentic
era~\citep{kwon2023pagedattention}, and this economic pressure, compounded with the
capability pressure that pure SSMs cannot solve exact
retrieval~\citep{jelassi2024repeat,arora2023zoology}, pushes the industry from both ends
toward the middle.

\subsection{The relation to Shannon: the price of the second convergence}
The hybrid does not violate Shannon~\citep{shannon1948mathematical}; under strictly equal
$B$ every architecture hits the same floor~\eqref{eq:floor}. It wins by meeting two
necessary conditions at an economical memory coefficient.

\begin{proposition}[The hybrid is access-complete at a reduced servable-memory coefficient]
\label{prop:minimal}
Solving Newton's apple requires \textup{(i)} $\Omega(Nb)$ indexable bits \emph{and}
\textup{(ii)} a state channel composable at $O(1)$ per-token update. A pure SSM lacks (i) at any
$o(Nb)$ state (unconditional, Theorem~\ref{thm:shannon}); a pure Transformer lacks (ii) even at
$\Theta(L)$ memory (conditional on $\mathrm{TC}^0\neq\mathrm{NC}^1$,
Proposition~\ref{prop:circuit}). Among access-complete resource profiles
(Definition~\ref{def:profile}), the hybrid attains both necessities while \emph{reducing} the
servable-memory coefficient: it keeps the $\Theta(L)$ index growth class that exact retrieval
demands, but multiplies its coefficient by $\rho\ll1$ (servable KV per token
$\mu\!\to\!\rho\mu$), while carrying composition on an $O(1)$-update state channel. Two
overstatements we avoid: this is not an asymptotic-order separation, since $O(\rho L)$ and
$O(L)$ share a growth class (the advantage over a full-attention access-complete baseline is
the coefficient, bandwidth, and serving economics of Definition~\ref{def:profile}; the
advantage over a pure SSM is the categorical one that the SSM is not access-complete at
all); and we do not claim the coefficient is minimal, since optimality would require a
workload-and-hardware optimization problem we have not formulated. What the census shows is
that current production designs operate in this reduced-coefficient region.
\end{proposition}

CCH and Shannon are complementary: Shannon sets the capacity ceiling; CCH says how to raise
it to the task's required height at minimal memory cost. The first convergence
(representation) is free because it consumes no channel capacity; the second (capability)
has a price, the $O(\rho L)$ index channel, and compositional emergence is the toll gate
that forbids freeloading capability by enlarging a pure architecture.

\section{Evidence: Is Capability Converging?}
\label{sec:regularities}

CCH is a hypothesis about the future, but it makes checkable claims about the present. The
division of labour in this section follows Table~\ref{tab:claims}: the four regularities
R1--R4 are consistency evidence for the convergence \emph{conjecture} (they do not
discriminate it from a pure-economics account), while the pre-registered controlled tests
that follow them are the tests of the access-completeness \emph{principle}, where the
falsifiable content lives. We state R1--R4 just as PRH stated ``representations are
converging,'' as a series of empirical observations.

\subsection{Four regularities in production architectures}

\paragraph{R1 (memory-convergence regularity).} The KV-memory \emph{coefficient} of
production architectures is falling within the same $\Theta(L)$ growth class, from $\mu$
toward $\rho\mu$ with $\rho\ll1$. A census of $N{=}140$ released architectures (2023-07 to
2026-05; sixteen distinct mixers), reading $\rho \equiv
L_{\mathrm{global}}/L_{\mathrm{total}}$ from primary-source configs
(Appendix~\ref{app:census}), shows: of $61$ genuine hybrids ($\rho\in(0,1)$), median
$0.167$, mean $0.207$, with $57\%$ in $[1/12,1/4]$, a band spanning $17\%$ of the axis
($3.4\times$ enrichment over uniform; $2.3$--$3.3\times$ against design-menu nulls,
$1.45\times$ against a scale-invariant null). Per-token KV memory spans ${\approx}160\times$ across
the census (Appendix~\ref{app:kv}); the most controlled evidence is same-vendor at fixed
recipe (Granite dense vs.\ hybrid $5.0\times$; Qwen3 vs.\ Qwen3-Next $10.7\times$). The
clustering is cross-sectional: independently designed hybrids land in a narrow band nothing
forces them into, stable since 2024 with per-cohort dispersion trending down
($\sigma = 0.148/0.119/0.067$) but not significantly (permutation $p\approx0.06$), so we
claim the concentration and treat the tightening as suggestive. Restricted to
production-frontier models (dropping distillation conversions and the on-device edge line by
the explicit rule of \texttt{experiments/census\_frontier.py}), the band tightens relative to
the full set: median $\rho = 0.125$, $\sigma = 0.109$ (vs.\ $0.135$), $62\%$ in band (vs.\
$57\%$), with recent flagships at $\rho \in \{{\approx}1/8,
{\approx}1/4\}$~\citep{lieber2024jamba,blakeman2025nemotronh,de2024griffin,minimax2025,granite2025,qwen3next2025,kimi2025linear}
(realized counts deviate from the design fractions by layer rounding, e.g.\
Kimi-Linear's headline $3{:}1$ is $7/27 = 0.259$).

\begin{figure}[t]
  \centering
  \includegraphics[width=0.82\linewidth]{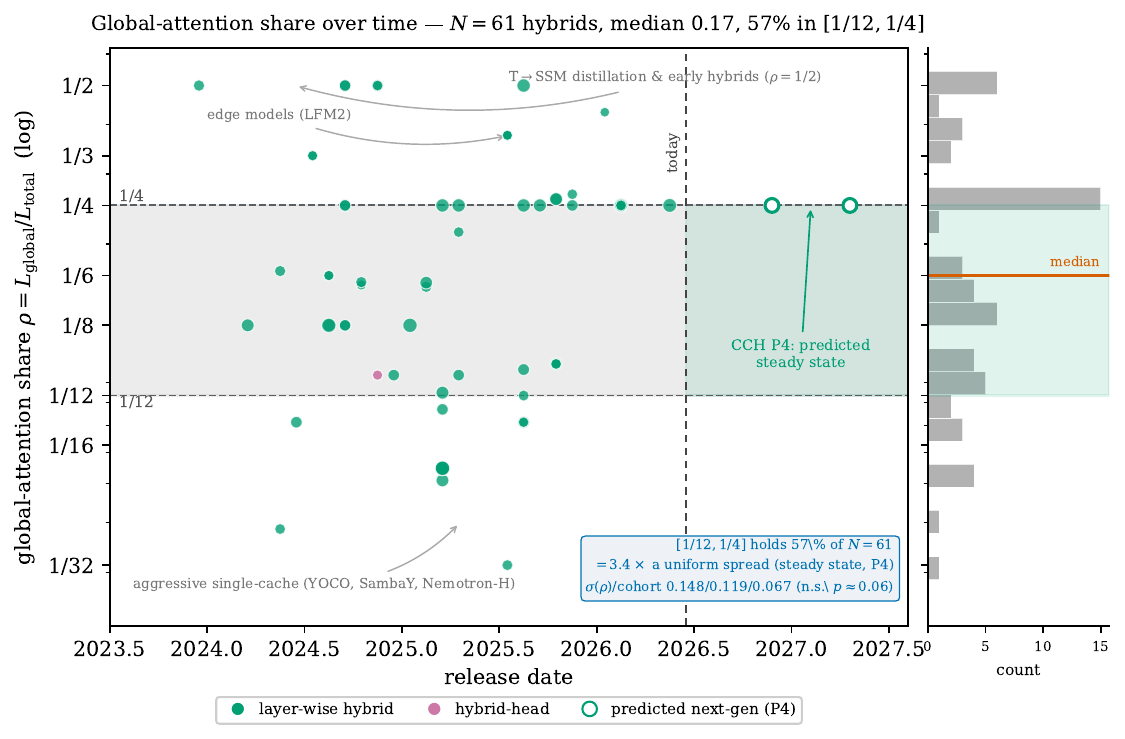}
  \caption{The global-attention share $\rho$ against release date ($N=61$ hybrids;
  source-verified). The band $[1/12,\,1/4]$ (shaded): median $0.167$, $57\%$ in band, a
  $3.4\times$ enrichment over uniform, behaving as a steady state (membership stable since
  2024, dispersion $0.148\!\to\!0.119\!\to\!0.067$ but $p\approx0.06$). Above-band outliers
  are distillation sweeps and edge models; below-band, single-shared-cache designs (YOCO,
  SambaY, Nemotron-H). Right of ``today'': the corridor of Prediction~\ref{pred:band}
  (open markers hypothesised).}
  \label{fig:rho}
\end{figure}

\paragraph{R2 (capability-fingerprint regularity).} Under a fixed budget, hybrid and dense
models are on par on short-context benchmarks (representation has converged), while
differences concentrate in verbatim long-range retrieval, state tracking, and many-shot
in-context learning~\citep{arora2023zoology,jelassi2024repeat}. A clean signal: distilling
a Transformer into an SSM leaves language nearly intact but collapses exact retrieval, and
restoring a small fraction of attention heads restores it; conversely, silencing the
attention layers of production hybrids collapses retrieval to zero while $15\%$ of heads
suffice~\citep{michalak2025someattention}, and mechanism analysis finds long-range
retrieval carried by full attention, the efficient layers shaping how fast retrieval heads
form~\citep{rethinking2026efficient}. What is missing is not representation but index, the
most direct empirical reading of the PRH/CCH
boundary~\citep{convergence2026understanding} (Figure~\ref{fig:comm}).

\paragraph{R3 (two-sided failure / compositional superiority regularity).} The evidence has
an Anna-Karenina structure: all access-complete models are alike, each pure architecture
fails in its own way. Aggressive linear variants collapse on high-order multi-hop reasoning
(cf.~\citealp{merrill2024illusion}); under matched recipes a hybrid is no weaker than full
attention across short-, long-context, and post-RL regimes~\citep{kimi2025linear}; an 8B
comparison at matched data finds the Mamba-2 hybrid matching or exceeding both pure families
while the pure SSM lags on in-context recall and copying~\citep{waleffe2024empirical}. The
pure SSM fails at recall, the pure Transformer at state, the projection of the three walls,
independently observed on synthetic retrieval~\citep{pantazopoulos2026retrievit}. The census
makes this quantitative (Figure~\ref{fig:capability}): every hybrid with a published RULER
number scores $84$--$95\%$ (Kimi-Linear $84.3$@128K, Jamba-1.5-Large $93.9$@256K, MiniMax-01
$94.7$, Qwen3-Next $91.8$), matching dense baselines (Qwen3-32B $93.7$, Llama-3.1-8B $77$),
whereas pure-compressive families collapse (xLSTM-7B ${\approx}\,20$) or publish no number
(consistent with the pattern, not a negative measurement). The Mamba2InLlama distillation is
the closest to a controlled comparison, a same-base correlation not a clean intervention: as
retained attention rises $0 \to \tfrac{1}{8} \to \tfrac{1}{4} \to \tfrac{1}{2}$, MT-Bench
rises ($5.64\!\to\!7.32$) and AlpacaEval-LC ($14.5\!\to\!26.8$, the $\rho{=}\tfrac{1}{2}$
hybrid exceeding the teacher's $22.9$). R2--R3 are conditional within $\Rgm$.

\begin{figure}[t]
  \centering
  \includegraphics[width=\linewidth]{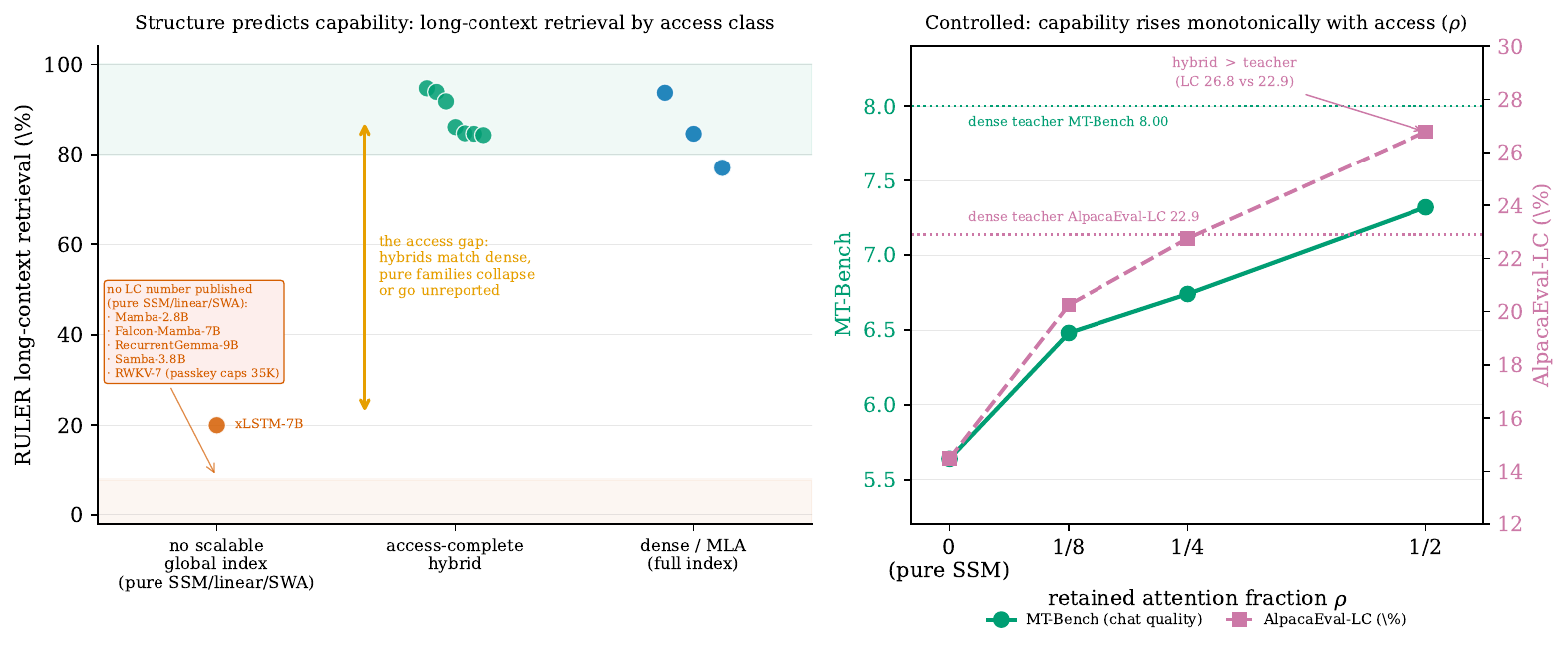}
  \caption{Published long-context retrieval stratifies by access class (the CCH analogue of
  PRH's Figure~4), the relation necessary-condition-shaped: a scalable index does not by
  itself predict the score, but no model without one posts a high score. \textbf{Left:}
  RULER retrieval by class; access-complete hybrids (green) match dense baselines (navy),
  no-index families (red) collapse (xLSTM-7B $20$) or publish no number (shown
  \emph{unmeasured}, a gap not a negative measurement; passkey-only and unreported models on
  the bottom rail). \textbf{Right:} the Mamba2InLlama staircase (fixed Llama-3-8B base,
  \citealp{wang2024mambainllama}): capability rises monotonically with retained $\rho$,
  exceeding the teacher on AlpacaEval-LC by $\rho{=}\tfrac{1}{2}$. Per-model data in
  Appendix~\ref{app:capdata}.}
  \label{fig:capability}
\end{figure}

\paragraph{R4 (form-independence regularity).} The forms realizing the hybrid are
diversifying while the access structure stays
isomorphic~\citep{lieber2024jamba,blakeman2025nemotronh,yang2024gla,sun2023retnet,peng2023rwkv,arora2024based}.
Across the census, at least sixteen distinct compression mechanisms (Mamba-1/2/3, Gated
DeltaNet and successors, Kimi Delta Attention, lightning/linear attention, RG-LRU,
RWKV-5/6/7, JetBlock, Hyena, short gated convolutions, MLA latent compression) appear across
seven-plus organizations, yet the frontier survivors all pair a compressive $O(1)$-state
channel with a scalable index (Figure~\ref{fig:sankey}: fourteen operators fan into four
outcomes, and Mamba-2, Gated DeltaNet, RWKV appear in both the access-complete column and
the degenerate corners, so structure, not operator, predicts membership). One qualification:
integration topology still matters within the class (sequential vs.\ parallel favor shorter
vs.\ longer contexts~\citep{lee2025understanding}), so access-completeness is a necessary
condition, not capability equivalence among realizations.

\begin{figure}[t]
  \centering
  \includegraphics[width=0.72\linewidth]{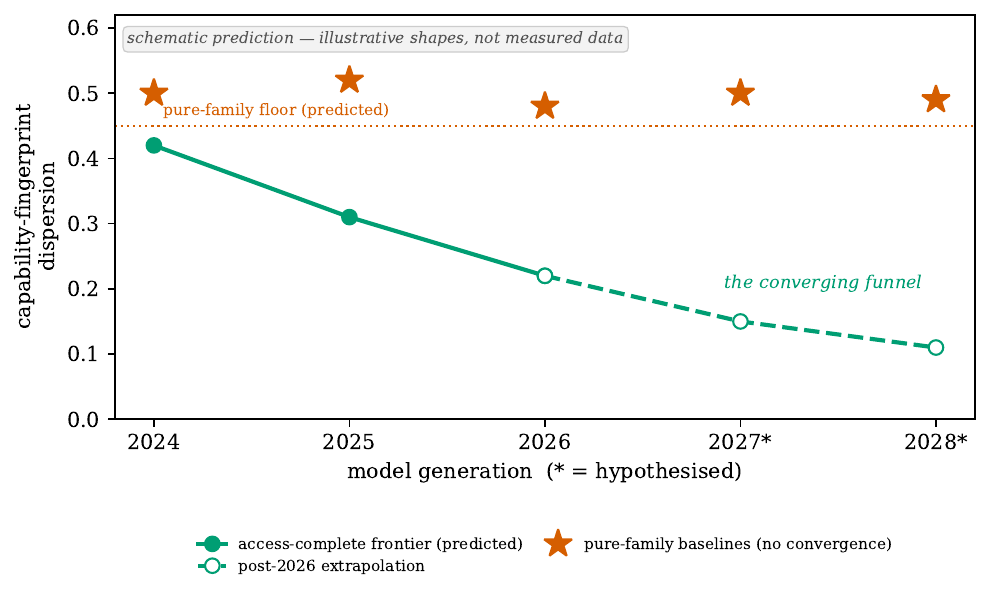}
  \caption{CCH's ontological prediction (Prediction~\ref{pred:fingerprint}): frontier
  access-complete models narrow on a capability fingerprint generation by generation, while
  pure-family baselines stay at a characteristic distance at any scale.
  \emph{Schematic, not measured.}}
  \label{fig:fp}
\end{figure}

These regularities sketch a picture consistent with (not demonstrating) an industry sliding
toward the access-complete class; ``class'' is deliberate, a design pattern multiply
realizable by distinct forms, not a single canonical architecture.

\paragraph{What R1--R4 do and do not settle.} R1--R4 are all compatible with a purely
economic account ($\Theta(L)$ KV is unservable, so the field migrates to whatever is
cheapest at acceptable quality), which does not distinguish ``the hybrid is strictly more
capable'' from ``hybrids are merely cheaper at equal capability.'' The migration is
therefore consistent-with evidence, not a discriminating test. What discriminates is
Predictions~\ref{pred:floor}--\ref{pred:comm}, above all the scissors gap and the
bifurcation: only there does strict superiority diverge from ``pure architectures are
already good enough.''


\subsection{Pre-registered small-scale tests}
\label{sec:experiments}

We tested P1--P3 (plus a model-level dissociation) under a protocol frozen before the data
existed: pass thresholds, one fitted constant, explicit falsification clauses, thirteen
timestamped amendments, no post-hoc exclusions. Everything runs inside $\Rgm$; trained
models are 0.2--1.1M parameters, public checkpoints $\le 3$B, so these are mechanism-level
tests (existence-level for the single-axis walls, training-reliability-level for the
conjunction), not frontier demonstrations. Full detail in
Appendix~\ref{app:experiments}; the scorecard is 11 supported, 7 partial, 1 failed of 19.

\textbf{P2, the bifurcation (Exp.~C; Figure~\ref{fig:nhgrid}).} Five of six predictions held
outright, including the starred one: $\beta\in(0,1)$ DeltaNet cannot even fit composed $S_5$
at 5$\times$ budget, while one reflection ($\beta\in(0,2)$) suffices for the swap stream at
$6.4\times$ training length. The $n_h$ ladder is a capacity result: only $n_h{=}4$ solves
the general stream at $T{=}256$ ($n_h\le 2$ never exceeds 0.022), but reliable trainability
fails an expanded-seed criterion (3/8), so we claim existence, not reliability.

\textbf{P1, the scissors (Exp.~B; Figure~\ref{fig:fano}).} Every pure-state error curve
tracks the $b(1-1/x)$ floor under a single fitted effective capacity ($p_{\mathrm{eff}}$,
consistent with the floor, not an independent verification of $\Hch(\Sigma)$), and the
hybrid holding the same 64-scalar state scores $0.994$ vs.\ $0.000$ at $N{=}128$. The
comparison holds the recurrent state fixed and \emph{adds} the index, so the hybrid's total
state is strictly larger, which is the point: the floor can only be satisfied by paying. The
collision control (amendment 8) pins the hybrid to the exact ambiguity floor $c/(c{+}1)$ at
zero code distance, so its advantage is purchased by separability; a window-only index is not
rescued ($0.995$), and 2-layer full attention is not an upper bound at this scale ($0.91$).

\textbf{The dissociation (Exp.~A; Figure~\ref{fig:dissoc}).} On Pythia/Mamba/RWKV (same
corpus and tokenizer, 70M--3B; observational, so we claim the pattern, not clean causal
attribution), PRH replicates while retrieval stratifies by access structure in all ten
matched-scale cells. The flagship cell (2.8B, $N{=}16$, $n{=}1000$) meets the frozen 30pp
margin ($32.0\pm3.3$pp); higher-$N$ and 1.4B cells fall below, so A-ii is partial with the
miss quantified. Alignment does not predict the gap (Spearman $+0.48$).

\textbf{P3, failed and re-operationalized (Exp.~D).} Commensurability failed, direction
reversed, across every control: cross-channel alignment \emph{decreases} with training and
success (random initialization, CKA $0.993$, is the ceiling). Reported as failed; the
measured story is channel differentiation, and P3$'$ (amendment 10) is a distinct follow-up,
not a retroactive rescue.

\textbf{The conjunction witness (Exp.~E; Figure~\ref{fig:witness}, Appendix~\ref{sec:expE}).}
A composite task writes
bindings [address $\to$ value], drives the referent by an $S_5$ stream, and queries the
value at the \emph{composed} referent (the primitive is due
to~\citealp{reasoning2026primitives}; ours is its pre-registered measurement). The witness
criterion is met as a \emph{training-reliability} separation: at the predicted cell
($T{=}256$, $N{=}16$) the recurrent-front hybrid passes on all seeds while no other arm
meets the criterion, but one pure-recurrence seed also solves it ($0.995$, the load being
far inside pure-state capacity), so existence-level separation would need supra-capacity
loads. There we measure the graded margin (hybrid's worst seed $>$ pure's best at $N{=}64$,
$0.706$ vs.\ $0.654$; $+0.53$ at the hardest cell). Arm-level: the attention-only control
never fits even in-length (nor at doubled budget, 9c, not a budget artifact); the
$\beta\in(0,1)$ hybrid loses both axes at once (the composed address couples them); the
window-only hybrid survives only where its state carries the load. Channel ablations (D-ii)
find joint necessity and exact value-side localization ($1.0$ vs.\ $0.004$); the frozen
differential sub-criteria failed because the conjunction couples the axes. A scope note:
CCH's necessity claim is computational, not mechanistic, and~\citet{rethinking2026efficient}
find efficient layers partly shape attention's retrieval heads, matching our bring-up
finding that the division of labour required a deeper recurrent front.

\begin{figure}[t]
  \centering
  \includegraphics[width=\linewidth]{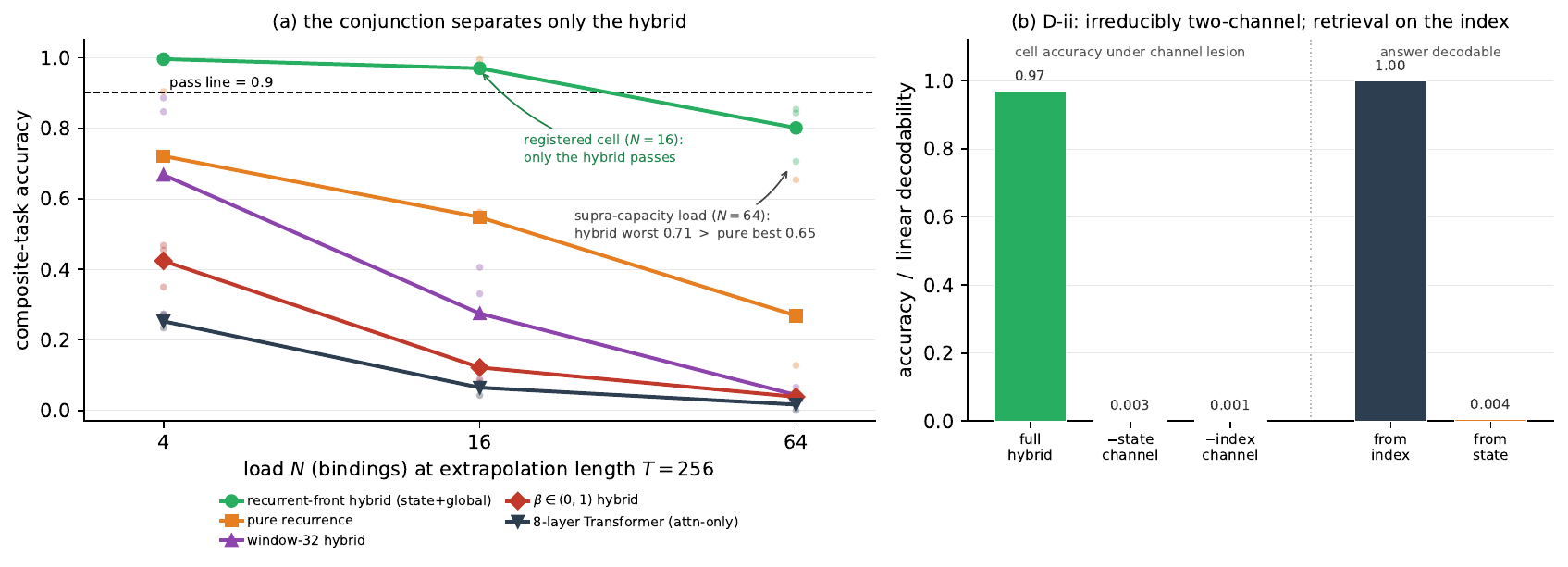}
  \caption{The conjunction witness, \emph{measured} (Experiment~E, Appendix~\ref{sec:expE};
  the composite [address$\to$value] $\times$ $S_5$ task, 3 seeds/cell). \textbf{(a)} Accuracy
  vs.\ load $N$ at the long-extrapolation length $T{=}256$ (mean lines, per-seed dots, pass
  line $0.9$): the recurrent-front hybrid (green) passes the registered cell ($N{=}16$) on
  every seed while no other arm meets the criterion; at the supra-capacity load $N{=}64$ the
  hybrid's worst seed ($0.706$) still exceeds the pure recurrence's best ($0.654$), a graded
  margin. This is a training-reliability separation, not an existence-level one (one
  pure-recurrence seed solves the $N{=}16$ cell). \textbf{(b)} D-ii channel complementarity
  (P3$'$): silencing either channel collapses the registered cell (the learned solution is
  irreducibly two-channel), and the answer is linearly decodable from the index mixer
  ($1.0$) but not the state mixer ($\approx0.004$): retrieval localizes to the index
  channel, as the mechanism predicts.}
  \label{fig:witness}
\end{figure}

\section{Discussion: Alternative Explanations}
\label{sec:alternatives}

The account above should be tested against its strongest rivals. We state four alternative
explanations in their best form, respond to each, and, for each one, say what evidence would
make it win.

\paragraph{AV1: better parameterizations, not more channels.} \emph{The DeltaProduct line
shows the state channel's expressivity is still improving; negative eigenvalues unlocked
state tracking; further parameterization advances will keep absorbing capabilities, and the
index channel is a crutch that engineering will outgrow.} Response: expressivity repairs the
circuit wall, not the Shannon wall. Capacity is a counting bound
(Theorem~\ref{thm:shannon}) that no parameterization of an $o(Nb)$ state can evade, and the
two gates are measured to act independently: expressivity gates tracking with capacity held
fixed (Experiment~C), capacity gates retrieval with expressivity held fixed (Experiment~B),
and the collision control shows that even a working global index buys nothing beyond what
the task's information structure licenses. AV1 wins if a pure-state architecture holds error
$\le 0.1$ at load ratio $x \ge 8$ with length generalization,
Prediction~\ref{pred:floor}'s falsification clause, open to anyone.

\paragraph{AV2: chain-of-thought and tools dissolve the walls.} \emph{Serial computation
escapes $\mathrm{TC}^0$~\citep{li2024chain}; retrieval tools externalize the index; the
industry ships both; a fixed-budget no-CoT regime is an artifact.} This is the strongest
alternative, and our response is scoping, not denial. First, $\Rgm$ is where per-token
economics live: CoT multiplies per-query cost by the reasoning length, and a tool call buys
the index back at latency and orchestration cost; the escape routes are real and priced.
Second, the walls relocate rather than dissolve: an agent's accumulated context is itself a
stream, the tool's store an index channel, so the question reappears one level up. Third,
the empirical picture with CoT is genuinely mixed, and we cite it as counter-evidence:
\citet{reasoning2026primitives} find no uniform hybrid advantage once reasoning tokens are
enabled, while their division-of-labour split (hybrid ``Think'' models more robust on
sequential state updates, Transformer ``Think'' on flat retrieval) is the part bearing on
the two-channel mechanism. AV2 wins if the witness family is solved at matched per-query
cost by a pure architecture plus CoT/tools, an economics-matched comparison we have not run
and flag as open.

\paragraph{AV3: it is economics all the way down.} \emph{R1--R4 are fully explained by
serving cost; hybrids are merely cheaper at equal capability; CCH's ``capability'' content
lives only in synthetic separator tasks.} Response: we agree about R1--R4 (the
identification limit is stated in Section~\ref{sec:regularities}), which is exactly why
CCH's bet is on Predictions~\ref{pred:floor}--\ref{pred:comm}, where the measured content
is a capability fact, not a cost fact ($0.994$ vs.\ $0.000$ on the same 64-scalar state).
Nor is the witness gerrymandered: it is the minimal conjunction of two loads routinely
present in long-horizon agentic work. AV3 wins if witness-type loads have vanishing measure
in natural workloads (Limitation~(iii)); the incidence measurement that would decide this
has not been run by anyone, including us (Appendix~\ref{sec:collnat} is a first step on the
separability half).

\paragraph{AV4: compress the index instead of pairing it.} \emph{MLA-style latent
compression (DeepSeek-V3) cheapens the index at $\rho{=}1$; perhaps the endgame is an
ever-cheaper pure-index architecture, no state channel.} Response:
Definition~\ref{def:profile} separates growth class from coefficient, and latent
compression improves only the coefficient (at fixed latent width MLA keeps $\Theta(L)$
index capacity, so Theorem~\ref{thm:shannon} binds it only under $o(L)$ growth); the census
shows MLA KV per token still an order of magnitude above layer-wise hybrids'. Meanwhile a
pure-index stack holds no $O(1)$-composable state, so the circuit wall stands, and
compression aggressive enough to reach $o(L)$ meets Theorem~\ref{thm:shannon} head-on,
surrendering exactness at high load, i.e.\ becoming a state channel by another name. AV4
wins if a $\rho{=}1$ architecture with sublinear servable-state growth passes both axes
with length generalization, refuting the channel dichotomy.

\section{Predictions, Limitations, and Implications}
\label{sec:predictions}

A hypothesis earns its standing by being falsifiable. Figures~\ref{fig:fano}
and~\ref{fig:nhgrid} now show measured results; Figure~\ref{fig:comm} remains schematic (its
test failed, direction reversed) and Figure~\ref{fig:fp} untested, in contrast to the
source-verified census figures.

\subsection{Falsifiable predictions}

\begin{enumerate}[leftmargin=1.8em,itemsep=3pt,label=\textbf{P\arabic*.}]
\item\label{pred:floor} \textbf{The information floor.} On Newton's-apple tasks, pure
architectures' error rises along the floor~\eqref{eq:floor} with load $x_s=Nb/B_s$ and
saturates, while access-complete hybrids (carrying the $\Omega(Nb)$ load on a separate index)
hug zero (Figure~\ref{fig:fano}), the decisive divide from ``pure architectures are good
enough.'' The falsifier is stated in \emph{nominal} terms so it cannot be evaded by
refitting an effective capacity: if any architecture with $o(Nb)$ nominal query-time state
and no scalable index maintains low error at high nominal load with length generalization,
CCH's characterization fails.
\item\label{pred:bifurcation} \textbf{State-tracking bifurcation.} A hybrid's $S_5$-tracking
should depend on state-channel expressivity (reflections per token and eigenvalue range:
$\beta\in(0,1)$ no reflection vs $\beta\in(0,2)$ with)~\citep{grazzi2025deltaproduct}: the
grid (Figure~\ref{fig:nhgrid}) bifurcates sharply, expressivity-poor hybrids failing $S_5$
extrapolation, sufficient ones succeeding.
\item\label{pred:comm} \textbf{Channel commensurability.} In solving hybrids, cross-channel
alignment should exceed that of failing hybrids and correlate with success
(Figure~\ref{fig:comm}). \emph{Tested and failed} (Appendix~\ref{sec:expD}), direction
reversed. The re-operationalization P3$'$ (amendment 10) finds joint necessity and exact
value-side localization ($1.0$ vs.\ $0.004$), the differential sub-criteria failing because
the conjunction couples the axes. Verdict: P3 failed; P3$'$ partially established, not a
retroactive rescue.
\item\label{pred:band} \textbf{Steady state of the share band.} If R1's band is set by
inference-economics optimization, it should be logarithmically insensitive to cost
parameters: a steady state, not a transient (Figure~\ref{fig:rho}). A long-term stable
mainstream model outside the band would falsify the economic parameterization.
\item\label{pred:fingerprint} \textbf{Fingerprint narrowing.} The dispersion of frontier
access-complete models on a capability fingerprint should narrow generation by generation
(Figure~\ref{fig:fp}); pure-family models should remain at a characteristic distance at any
scale. A non-converging control metric should be tracked to guard against confirmation bias.
\end{enumerate}

The attractor claim also has an aggregate falsifier: a production-frontier generation in
which pure-family models close the retrieval and state-tracking gap with access-complete
hybrids at matched serving cost (R1's band dissolving, R3's stratification flattening)
would falsify CCH as a whole.

\subsection{Postscript (added in revision): two frontier releases after registration}
\label{sec:postreg}

Between the first public version of this paper (2026-07-14; census frozen 2026-05) and
this revision, two frontier flagships appeared. They bear directly on
Prediction~\ref{pred:band} (the capability-fingerprint counterpart,
Prediction~\ref{pred:fingerprint}, awaits measurement), and we report them here rather
than inside R1: every registered census statistic in this paper is unchanged, and the new
rows ship separately, marked post-registration
(\texttt{figs\_en/census\_data\_postreg.json}, audited by
\texttt{experiments/census\_postreg.py}).

\textbf{Kimi-K3~\citep{kimi2026k3}.} The largest open-weight release as of this revision
(announced 2026-07-16; weights and configuration public 2026-07-26): a
2.8T-parameter LatentMoE flagship (104B activated, 16 of 896 experts) with a 1M-token
window. Its access structure, read from the released configuration: 93 layers, 69 Kimi
Delta Attention linear layers carrying the $O(1)$ state and 24 gated-MLA layers carrying
the global index, $\rho = 24/93 = 0.258$, the same every-fourth-plus-final-layer pattern
as Kimi-Linear-48B ($7/27 = 0.259$), marginally above the literal band edge and at the
2026 modal value $\rho\approx 1/4$. The MLA layers are NoPE with a 512-dimensional cached
latent, so the $\Theta(L)$ cache is $22.9$\,GiB per 1M tokens (bf16), equal to
Qwen3-Next-80B's and ${\sim}13\times$ under dense Llama-3.1-70B. Three readings. The
aggregate falsifier did not fire: the flagship did not abandon the index channel. The
share attractor held out of sample: a distinct operator pair lands on the same structure at
the same share, the same vendor's design scaled $58\times$ in total parameters with the
access structure unchanged, and the 2026 cohort dispersion tightens from $0.067$ to
$0.062$ with K3 added ($n{=}7$, median still $1/4$). And the 1M-token window is served at
$\rho\ll 1$, the horizon economics of R1 at production scale: the conjunction of
Section~\ref{sec:newton}, compress what has structure and index what does not, deployed
as a product. Two scope notes the release sharpens: extreme expert sparsity compresses
the parameter axis, not the context axis, and is orthogonal to the walls; and K3's
attention residuals, layers reading earlier layers' residual streams through an attention
score, open an access channel along \emph{depth}, an axis our theory does not cover and
we mark as open.

\textbf{Qwen3.8-Max~\citep{qwen2026max}.} A 2.4T-parameter MoE flagship previewed
2026-07-19 with its attention layout undisclosed at this revision. We register a forward
prediction before disclosure, in two parts so the falsification boundary is unambiguous.
Primary: the published architecture will hold a scalable global index channel at minority
share ($0 < \rho < 1/2$); a pure-state or SWA-only disclosure at this scale would falsify
Prediction~\ref{pred:band} at the frontier. Sharp: $\rho$ will lie in $[1/12, 1/4]$ up to
the final-layer counting convention that places K3 at $0.258$ (that is,
$\rho \le 1/4 + 1/L_{\text{total}}$; the registered $57\%$/$62\%$ statistics keep the
literal closed interval, and this tolerance defines only the boundary of this prediction),
with point value $\rho \approx 1/4$, the modal share
of its lineage (Qwen3-Next and the Qwen3.5 family, gated DeltaNet hybrids at exactly
$1/4$). The primary part is the attractor claim; the sharp part is what the 2026 cohort
makes most probable.

\begin{figure}[!t]
  \centering
  \includegraphics[width=0.88\linewidth]{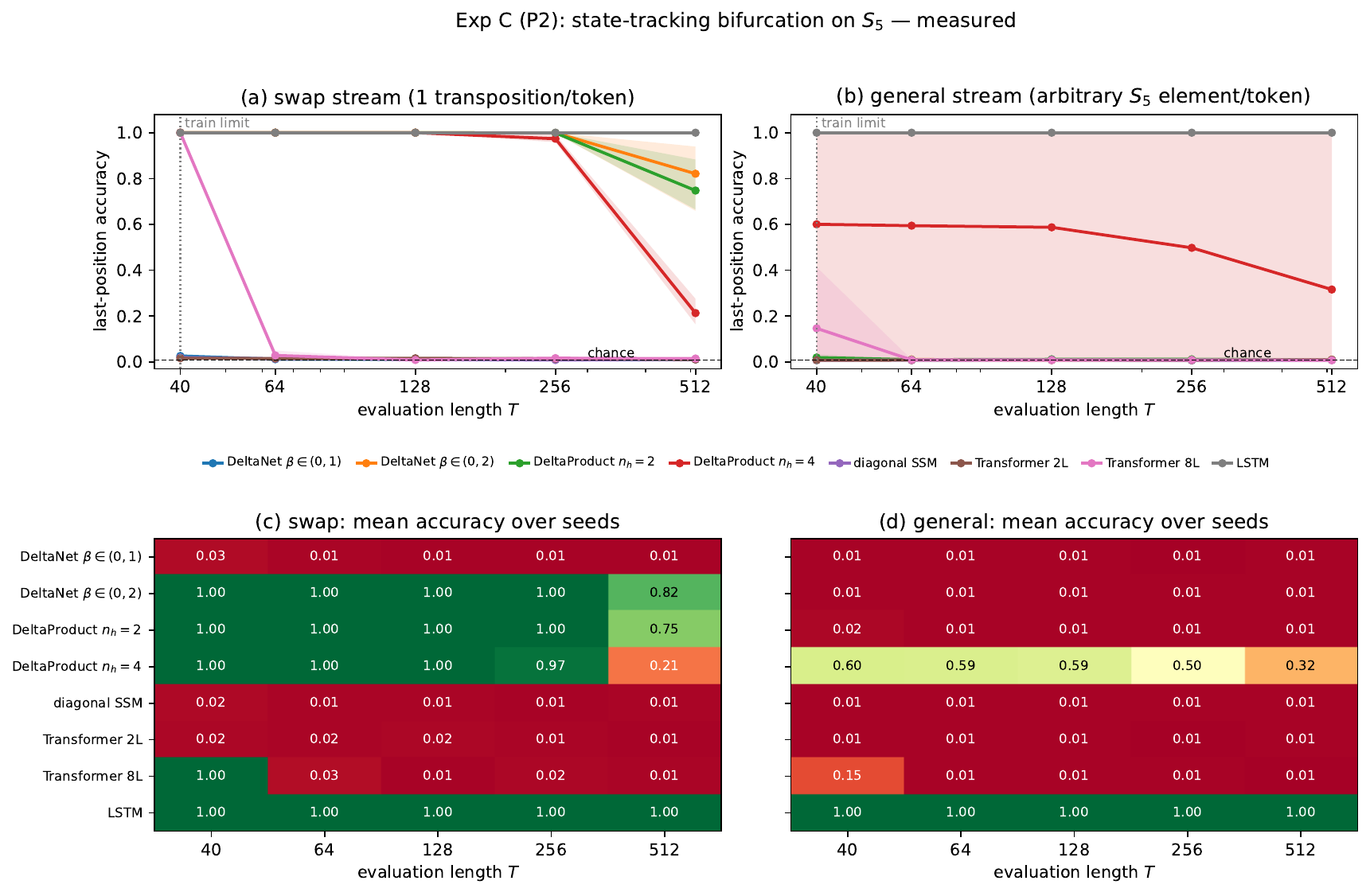}
  \caption{The state-tracking bifurcation, \emph{measured} (Experiment~C,
  Appendix~\ref{sec:expC}: running $S_5$ products, trained at $T\le 40$; 3 seeds per cell,
  8 for $n_h{=}4$ general; bands span seeds). \textbf{(a,b)} Accuracy vs.\ length on the
  swap and general streams. \textbf{(c,d)} The measured grid: the bifurcation appears at
  the pre-registered boundary for this configuration. $\beta\in(0,1)$ ($*$, the sharpest
  claim) cannot even fit the training distribution; one reflection ($\beta\in(0,2)$)
  suffices for swap at $6.4\times$ training length; the general stream is solved only at
  $n_h{=}4$ (3/8 seeds at $T{=}256$, a capacity rather than trainability result; one seed
  reaches $1.0$ at $T{=}512$); diagonal SSM and both Transformer controls fail
  extrapolation (the 8-layer, $4\times$-parameter Transformer hits $1.0$ in-length and
  collapses to $0.03$ at $T{=}64$); the LSTM passes everything.}
  \label{fig:nhgrid}
\end{figure}

\subsection{Limitations}
\label{sec:limitations}

\paragraph{(i) Regime boundary.} CCH is confined to the fixed-budget, no-CoT System-1
regime; chain-of-thought and tools are orthogonal escape axes trading time for
space~\citep{li2024chain}. CCH describes the limit of fast thinking, not all regimes.

\paragraph{(ii) Conditionality.} The circuit wall is conditional on
$\mathrm{TC}^0\neq\mathrm{NC}^1$, believed by nearly all but proven by none; we back it with
the unconditional Shannon wall, so even if the circuit wall fails, pure architectures remain
barred under a fixed budget.

\paragraph{(iii) Naturalness of the witness: the largest open risk.} The largest
uncertainty is not whether the separation exists but how often natural workloads
instantiate the joint load: Newton's apple is a worst-case construction, and $S_5$-level
composition is not frequent in everyday language. CCH's constraining force has a tail-risk
shape (the cumulative reference/world-state update of long agentic dialogue is
low-frequency but high-harm); we claim such tasks exist and are gated by architecture, not
that they have large measure. The same openness applies to Assumption~\ref{ass:sep}, on
which a first surface-form measurement (Appendix~\ref{sec:collnat}) now puts numbers:
entity-like keys are largely separable ($87\%$ at document scale) while numeric IDs
($51\%$) and code identifiers ($56\%$) are not, so the assumption is plausible for
entity-style loads and materially strained for morphologically regular keys; the
embedding-level measurement remains open.

\paragraph{(iv) From expressivity to learnability.} The construction is an existence
argument, not a learnability argument. Section~\ref{sec:experiments} bridges the gap at
small scale (gradient descent finds the hybrid solution; the bifurcation lands at the
registered boundary), but only empirically: gradient descent may find a functionally
equivalent manifold rather than the construction itself, and the bridge at frontier scale
is open.

\paragraph{(v) Scale extrapolation.} CCH's skeleton is scale-independent by design
(difficulty is set by load ratio, not absolute scale), but small-model long-context
behavior may differ from frontier-scale~\citep{hsieh2024ruler}, so scale-universality is an
extrapolation until large-scale reproduction. Two scale-dependent escape routes: a
frontier Transformer might approximate a state channel through extreme depth or sparse
routing (leaving $\Rgm$ in effect), and the separable codes Assumption~\ref{ass:sep}
requires might emerge only under particular data mixes; whether gradient descent at scale
induces either is open.

\subsection{Implications for measurement and reporting}
\label{sec:action}

If the account is right, several measurement and reporting practices should change. We give
concrete recommendations for four audiences.

\paragraph{For benchmark designers.} Treat access structure as a first-class reported
variable: (i) report each model's access class (pure-compressive / local-index-only /
access-complete / pure-index) alongside parameter count and context length; (ii) include a
fixed-budget, no-chain-of-thought track, the regime where the walls bind; (iii) include at
least one \emph{conjunction} task (verbatim retrieval composed with state tracking), since
that is exactly where pure families separate while single-axis scores look interchangeable.

\paragraph{For model developers.} Report the global-attention share $\rho$ and the
servable-memory coefficient $\mu$ (Eq.~\eqref{eq:kv}) in model cards. Our census had to
reverse-engineer both; they are two numbers, they determine serving economics, and if CCH
is wrong they are what a refutation needs.

\paragraph{For empiricists.} Run controlled same-data architecture comparisons at
increasing scale (in the style of~\citealp{waleffe2024empirical}) with the conjunction in
the evaluation set, pre-registered where possible; this paper's protocol is offered as a
template. The three open measurements we most want run, preferably adversarially: the
corpus-incidence measurement of witness-type loads (AV3); the economics-matched CoT/tools
comparison (AV2); and an equal-total-query-time-memory hybrid vs.\ full-attention
comparison with a control that actually trains (our 2-layer control did not, so
Experiment~B lacks a working pure-index upper bound).

\paragraph{For theorists.} The naturalness gap is the load-bearing open problem: when does
natural text supply Assumption~\ref{ass:sep}'s write-time separability? Concretely, estimate
the fraction of retrieval-relevant bindings whose keys fall within an $\varepsilon$-ball of
another's in the write-time embedding; CCH predicts hybrid advantage tracks the separable
fraction. We run the surface-form version ourselves (Appendix~\ref{sec:collnat}); the
embedding-level one is open.

\paragraph{The reporting checklist.} Concretely, the fields we ask to see reported:

\begin{table}[h]
\centering
\small
\begin{tabular}{@{}ll@{}}
\toprule
\textbf{Reported quantity} & \textbf{Meaning} \\
\midrule
$B_{\mathrm{state}}$ & compressive state capacity (scalars $\times$ precision) \\
$\mu$ & servable memory, bytes per token at query time (Eq.~\eqref{eq:kv}) \\
$\rho$ & global-index share $L_{\mathrm{global}}/L_{\mathrm{total}}$ \\
$\beta_q$ & query-time read bandwidth \\
access class & pure-compressive / local-index / access-complete / pure-index \\
fixed-budget score & long-context capability, no CoT, no tools \\
conjunction score & retrieval $\times$ state-tracking witness performance \\
\bottomrule
\end{tabular}
\end{table}

\section{Conclusion}
\label{sec:conclusion}

PRH proposed where representations go: toward the same $Z$. CCH proposes where capability
goes, and names its price: the index channel is the cost, compositional emergence the
barrier against freeloading it by scale, and the access-complete class the attractor, a
design pattern rather than a single architecture. The first convergence is given freely by
scale because it consumes no query-time access capacity (paid, of course, in data and
training compute); the second must be purchased by access structure, constrained by
Shannon's ceiling~\citep{shannon1948mathematical} and circuit
depth~\citep{merrill2023parallelism} and driven by inference economics. Newton's apple is
the joint requirement made vivid: maintaining the direction of a long-range question while
locking onto one distant binding at query time. Our conjecture is that machine intelligence
is converging along this axis: beyond representational convergence, capability convergence
is next. Stated as practice rather than prophecy: access structure should become a
first-class variable in how this community evaluates and designs long-context models
(Section~\ref{sec:action}), and
Predictions~\ref{pred:floor}--\ref{pred:fingerprint} are the terms on which CCH can be
dismissed.

\paragraph{Statement.} The theory's predictions have met their first
pre-registered tests at small scale. P1 is supported in two senses kept separate: the
scissors gap is measured outright, while the floor-shape collapse holds under one fitted
effective capacity, consistent with but not an independent verification of
Theorem~\ref{thm:shannon} at nominal capacity. P2 is supported as an expressivity boundary
in the registered configuration, with optimization brittleness disclosed. The conjunction
witness meets its criterion as a training-reliability separation; the joint-necessity claim
of Proposition~\ref{prop:nonpreserve} is thus supported existentially only through its two
single-axis projections, its end-to-end existence-level form open at supra-capacity loads.
P3 failed and is reported as failed; P3$'$ is partially established. We claim exactly that
much: mechanism-level evidence at $0.2$--$1.1$M trained parameters and $\le 3$B public
checkpoints, under frozen criteria with all misses disclosed, not confirmation at frontier
scale (Limitation~(v)). Industrial data are public or vendor-reported; census points later
than the date of writing are marked as extrapolations.

\appendix
\section{Explicit Construction for the Hybrid Crossing Newton's Apple}
\label{app:construction}

\begin{figure}[t]
  \centering
  \includegraphics[width=0.82\linewidth]{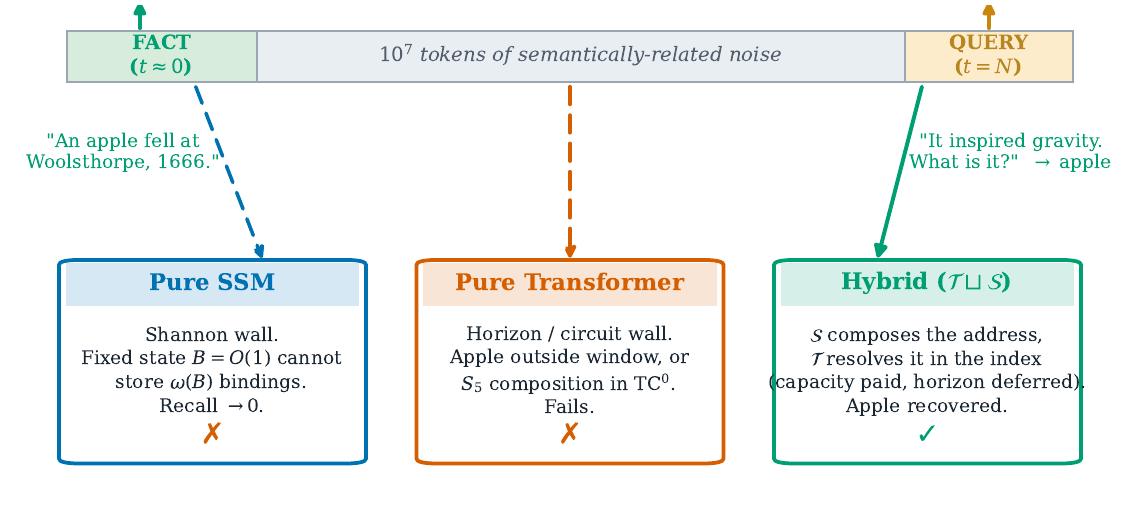}
  \caption{The Newton's-apple thought experiment in an infinite stream. A single fact planted
  at $t\approx 0$ must survive $10^7$ tokens of semantically related noise to be retrieved at
  the query $t=N$. A pure SSM is barred by the Shannon wall (Theorem~\ref{thm:shannon}); a
  budgeted Transformer by the horizon wall; an idealized Transformer by the circuit wall
  (Proposition~\ref{prop:circuit}). Only the access-complete hybrid crosses, and it does so
  by paying: the state channel composes the address, and a $\Theta(\rho L)$ index resolves it
  under Assumption~\ref{ass:sep}; the horizon is thereby deferred by $1/\rho$, not eliminated
  (Section~\ref{sec:cross}).}
  \label{fig:newton}
\end{figure}

This appendix gives the checkable-grade construction underlying Section~\ref{sec:cross}.

\paragraph{A.1 State encoding.} $r_t \in S_5$ is a $5\times 5$ permutation matrix $P_t$,
embedded in $d_{\mathrm{state}} = 25 + d_{\mathrm{code}}$ dimensions ($\mathrm{vec}(P_t)$
plus a transient anchor-codeword buffer used only while writing the current binding, so the
persistent state is $P_t$ alone); $P_0 = I$.

\paragraph{A.2 Per-token update.} ``Transpose $(i,j)$'' maps to a Householder reflection
$H_{ij} = I - 2uu^\top$, $u = (e_i-e_j)/\sqrt{2}$ (orthogonal, eigenvalue $-1$, hence
$\beta\in(0,2)$). A composite token ($\le 4$ transpositions) fills $n_h=4$ slots:
$P_t = P_{t-1}\prod_{k\le 4} H_k$, one DeltaProduct
step~\citep{siems2025deltaproduct}, explaining the staircase of Figure~\ref{fig:nhgrid}.

\paragraph{A.3 Discrete self-correction (the key to finite precision).} Each step is
followed by a snap nonlinearity (row-argmax re-binarization); the permutation set is its
attracting fixed-point set, single-step rounding error $<1/2$ spacing is absorbed, and any
constant precision suffices, bypassing $L\cdot 2^{-p}$ accumulation. One caveat: exact
row-argmax is not a standard DeltaProduct operation, so the construction's family is
``DeltaProduct blocks plus a pointwise saturating nonlinearity.'' For an approximate snap,
single-step accuracy alone does not give all-$L$ accuracy; what suffices is a contraction
toward the finite permutation set: if for some $\lambda<1$ the map sends each permutation's
$\varepsilon$-neighborhood into its $(\lambda\varepsilon+\delta)$-neighborhood with
$\delta<(1-\lambda)/4$, then $\|e_{t+1}\|\le\lambda\|e_t\|+\delta$ stays below half the
spacing uniformly in $L$; on the finite compact reachable set a temperature-sharpened
softmax realizes such a contraction (existence, not an explicit construction).
(Empirically moot at the scales tested: the trained models of Appendix~\ref{sec:expE} track
$S_5$ and extrapolate without any explicit snap.)

\paragraph{A.4 Anchor write and retrieval.} A binding is written as a Hopfield pattern
$\xi = [\mathrm{code}(\mathrm{key}) \oplus \mathrm{tag}(r) \oplus \mathrm{value}]$ with a
minimum-distance code ($d_{\mathrm{code}} = O(\log N + b)$); distractors keep spacing
$\delta_\kappa$ from codewords by Assumption~\ref{ass:sep}. The query
$q_L = W_q[\mathrm{vec}(P_L)\oplus\mathrm{probe}]$, with $\mathrm{probe}$ the queried key's
codeword and $\mathrm{tag}(r)$ a linear readout of $\mathrm{vec}(P_L)$, matches the unique
pattern agreeing on both key code and composed-referent tag; the value block is
code-orthogonal, contributing no cross-talk, and the margin $\Delta$ is inherited from
$\delta_\kappa$. Hopfield single-step error is $\le M e^{-\beta_H \Delta}$~\citep{ramsauer2021hopfield}
with prefactor $M = \Theta(\rho L)$ counting all stored patterns (not $N$); choosing
$\beta_H = O((\log L + \log\delta^{-1})/\Delta)$ gives per-query error $\le \delta/2$, so
the required score precision grows as $O(\log L)$, exactly the budget $\Rgm$ allows.

\paragraph{A.5 Resource accounting.} Memory $= O(d_{\mathrm{state}}^2) + \rho L d$. Compute
separates by channel (Definition~\ref{def:profile}): state update $O(d^2)$ per token
(independent of $L$), index append $O(d)$ per token, but the index read scans
$\Theta(\rho L)$ patterns at $\Theta(\rho L d)$ per query, so total serving cost has no
$O(1)$-per-token bound; $O(1)$ is the recurrent update, and $\rho$ is the reduced
coefficient. The hybrid realizes Newton's apple at $V \ge 1-\delta$ with memory
$O(d_{\mathrm{state}} + \rho L d)$, in the reduced-coefficient region of
Proposition~\ref{prop:minimal} (a smaller constant, not growth class; no optimality
claimed).

\paragraph{A.6 What is and is not $O(1)$.} $d_{\mathrm{code}} = O(\log N + b)$ does not
contradict ``$O(1)$ state'': the state is $O(1)$ \emph{in $L$} (a fixed number of
Householder products plus a snap, not growing with the stream), and its width scales with
task parameters ($O(\log N + b)$) but with no part of the extensive $\Theta(Nb)$ memory,
which the index bears. So ``$O(1)$ state'' means $O(1)$ per-token update with
$L$-independent width, and it is that $L$-independence that makes the Shannon wall bite.
The intended limit fixes $(N,b)$ and lets $L\to\infty$; coupling $N=N(L)$ grows the control
width as $O(\log N(L)+b)$, subextensive and harmless since $\Sch$ is defined by total state
$o(Nb)$. Experiments vary $N$ at fixed $L$, probing the load axis directly.

\section{Galois Closure Axioms, $\Solv$ Semantics, and Deferred Proofs}
\label{app:galois}

\subsection{The capability closure, via a Galois connection}
\paragraph{Notation.} Architecture families live in a universe $\mathfrak{U}$
($\mathfrak{a},\mathfrak{b}$; $\Tch,\Sch,\Hch$ reserved for the Transformer, SSM, and
hybrid families); task families live in $\mathcal{T\!sk}$ ($\tau$;
$\mathcal{Q}\subseteq\mathcal{T\!sk}$); no symbol denotes both.

Fix regime $\Rgm$: the resource profile of Definition~\ref{def:profile} held fixed as a
class. One clarification the shorthand can obscure: the budget pins the \emph{serving}
profile (state-update compute, memory coefficient, bandwidth), not the query-time read
compute of an index channel, which necessarily scales as $\Theta(\rho L)$ with the span it
indexes; demanding $O(1)$ read compute of every channel would exclude every access-complete
architecture by fiat and collapse CCH to the Shannon wall alone. What distinguishes
architectures inside $\Rgm$ is which memory-growth class and coefficient they pay for that
read. Define the \emph{solvability relation}:
$(\mathfrak{a},\tau) \in \Solv_{\Rgm}$ iff there exists a model sequence in $\mathfrak{a}$
reaching accuracy $\ge \tfrac{2}{3}$ on $\tau$ under $\Rgm$ with length generalization. This
relation induces a Galois connection between the powerset lattices
$(2^{\mathfrak{U}},\subseteq)$ and $(2^{\mathcal{T\!sk}},\subseteq)$ via the two polar maps
\begin{equation}
F(A) = \{\tau : \forall\, \mathfrak{a}\in A,\ (\mathfrak{a},\tau)\in\Solv_{\Rgm}\}, \qquad
G(\mathcal{Q}) = \{\mathfrak{a} : \forall\, \tau\in \mathcal{Q},\ (\mathfrak{a},\tau)\in\Solv_{\Rgm}\}.
\label{eq:polar}
\end{equation}

\begin{lemma}[Galois connection and closure]
\label{lem:galois}
The pair $(F,G)$ is a Galois connection: $\mathcal{Q} \subseteq F(A) \iff A \subseteq G(\mathcal{Q})$.
Consequently $\Capb := F \circ G$ is a \emph{closure operator} on $2^{\mathcal{T\!sk}}$,
satisfying \textup{(extensivity)} $\mathcal{Q} \subseteq \Capb(\mathcal{Q})$; \textup{(monotonicity)}
$\mathcal{Q}_1 \subseteq \mathcal{Q}_2 \Rightarrow \Capb(\mathcal{Q}_1) \subseteq \Capb(\mathcal{Q}_2)$;
\textup{(idempotence)} $\Capb(\Capb(\mathcal{Q})) = \Capb(\mathcal{Q})$.
\end{lemma}

\begin{proof}
Immediate from~\eqref{eq:polar}, since both sides of the biconditional assert
$\forall\mathfrak{a}\in A,\forall\tau\in\mathcal{Q},(\mathfrak{a},\tau)\in\Solv_{\Rgm}$; the
closure axioms are the standard properties of $F\circ G$ for any Galois
connection~\citep{ganter1999formal}. Full details and $\Solv$ semantics are given below.
\end{proof}

We write $\Cinf(\mathfrak{a}) := F(\{\mathfrak{a}\})$ for a single family's capability
closure: the tasks it eventually solves under unbounded scaling within $\Rgm$.
Lemma~\ref{lem:galois} is what licenses calling $\Cinf$ a closure (an extensive idempotent
map of a Galois connection), not a metaphor.

\paragraph{Deferred proof of Theorem~\ref{thm:shannon}.}
\begin{proof}[Proof sketch]
The bound is a data-processing inequality, not a communication-complexity result. Because the
bindings are independent, $I(V_I;I)=0$, so $I(V_I;\Sigma,I)=I(V_I;\Sigma\mid I)$; and
super-additivity of mutual information over independent sources (equivalently, sub-additivity
of conditional entropy) gives
$\sum_i I(V_i;\Sigma) \le I(V_{1:N};\Sigma) \le \Hch(\Sigma)\le B$. Averaging over the uniform
index, $I(V_I;\Sigma\mid I)=\tfrac1N\sum_i I(V_i;\Sigma)\le B/N$, which is~\eqref{eq:perelem}.
Since $\Hch(V_I\mid I)=b$, we obtain $\Hch(V_I\mid\Sigma,I)=b-I(V_I;\Sigma,I)\ge b-B/N=b(1-1/x)$,
which is~\eqref{eq:floor}. The log-loss floor is Gibbs' inequality; the error floor is Fano.
\end{proof}

\begin{remark}[Worst-case strengthening via \textsc{Index}]
\label{rem:index}
Equation~\eqref{eq:perelem} is average-case. A worst-case decision bound for a deterministic
compressive model embeds the recurrence $\Sigma_t=f(\Sigma_{t-1},x_t)$ into a one-way
protocol: the hidden state after the binding prefix is the single Alice$\to$Bob message of
\textsc{Index}, randomized one-way complexity $\Omega(N)$~\citep{kremer1999randomized}, and
Yao's principle~\citep{yao1977probabilistic} transfers the bound to the best model.
\end{remark}

Semantics: the $\Solv$ threshold is accuracy $\ge 2/3$ with $4\times$ training-length
extrapolation without decay; $\infty$ denotes the joint limit of parameters, data, and
length along the regime-preserving direction. The experimental operationalization is
deliberately stricter ($\ge 0.9$ at $6.4\times$); every fail verdict sits at chance, far
below $2/3$, so no verdict changes under the weaker threshold.

\begin{remark}[Two lattices, one map: a caution against over-reading the algebra]
\label{rem:lattices}
The inclusion~\eqref{eq:nonpreserve} is not the failure of a union-preserving homomorphism:
the antitone polar map obeys only $F(A\cup A')=F(A)\cap F(A')$, and the object mixes two
lattices ($\sqcup$ is the architecture-lattice join, a strictly larger interleaved family;
$\cup$ is ordinary task-lattice union). The precise content of Corollary~\ref{cor:nonjoin}
is that the composite family's reachable tasks strictly exceed the union of the
components', not that any clean algebraic identity is violated.
\end{remark}

\section{KV-Memory Accounting for Figure~\ref{fig:kvbar}}
\label{app:kv}

\begin{figure}[t]
  \centering
  \includegraphics[width=0.82\linewidth]{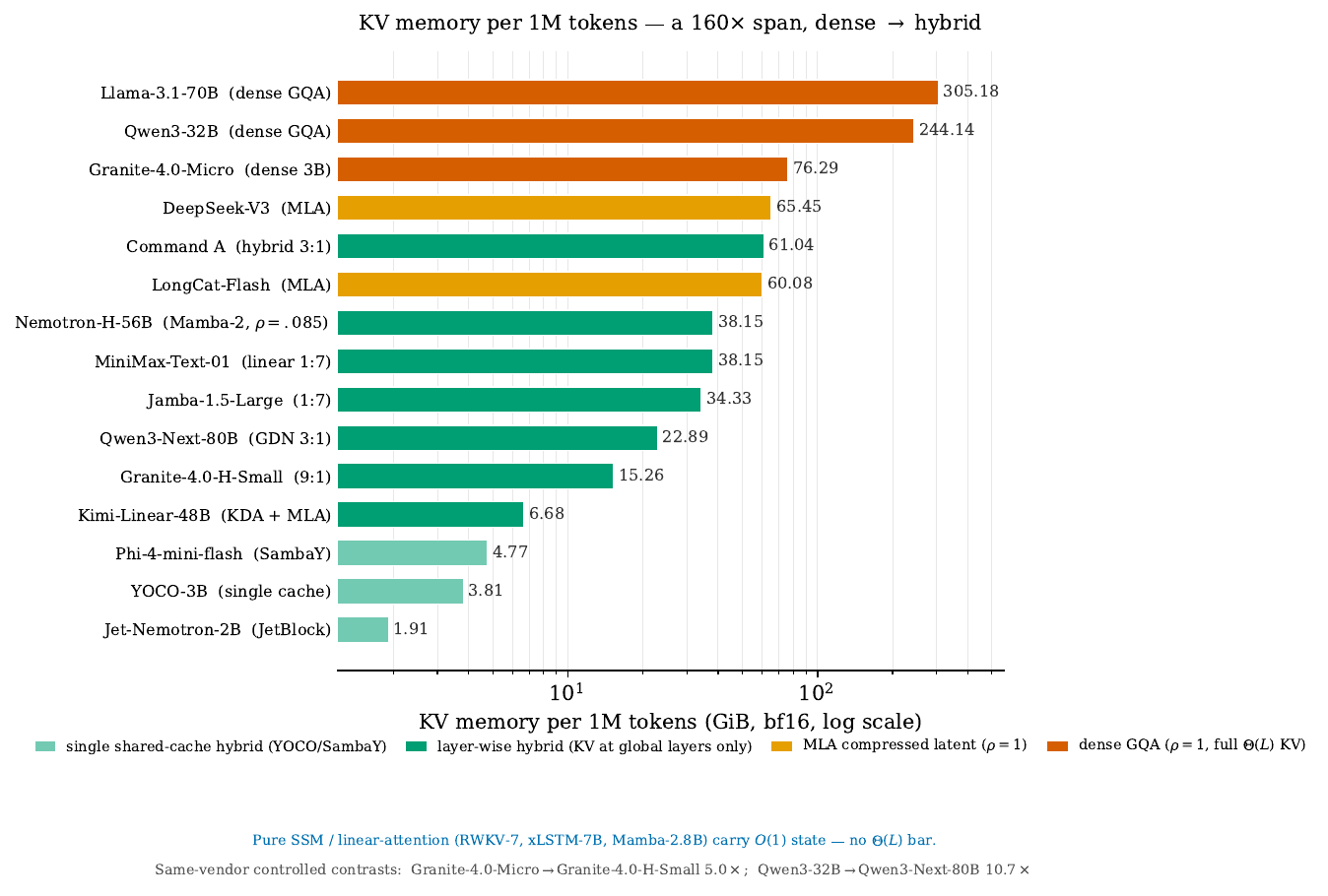}
  \caption{KV-memory per $1$M tokens (bf16, log scale) across $15$ census architectures,
  spanning ${\approx}160\times$ from Llama-3.1-70B ($305$\,GiB) to Jet-Nemotron-2B
  ($1.9$\,GiB): dense GQA (red) caches at every layer, MLA (amber) a compressed latent at
  $\rho{=}1$, layer-wise hybrids (green) only at global-attention layers; pure SSM/linear
  models have no $\Theta(L)$ bar. The memory-side signature of R1. Source-verified via
  Eq.~\eqref{eq:kv} (Table~\ref{tab:kv}); five headline figures re-derived to within
  $2.4\%$ of the originating labs' numbers.}
  \label{fig:kvbar}
\end{figure}

The per-token servable key--value state of an attention layer has size
\begin{equation}
\mathrm{KV\ bytes\ per\ token} \;=\; 2 \cdot n_{\mathrm{layers}}^{\mathrm{attn}} \cdot
n_{\mathrm{kv}} \cdot d_{\mathrm{head}} \cdot s,
\label{eq:kv}
\end{equation}
where $n_{\mathrm{layers}}^{\mathrm{attn}}$ counts \emph{global-attention} layers (linear/SSM
layers contribute $O(1)$, not $\Theta(L)$), $n_{\mathrm{kv}}$ the key/value heads,
$d_{\mathrm{head}}$ the head dimension, $s$ bytes per scalar ($s=2$ bf16). Table~\ref{tab:kv}
lists the config parameters (computed from model cards, not measured); a hybrid's
$\rho = n_{\mathrm{layers}}^{\mathrm{attn}}/n_{\mathrm{layers}}$ is the share of
Figure~\ref{fig:rho}.

\begin{table}[h]
\centering
\small
\caption{Configuration parameters entering~\eqref{eq:kv} for Figure~\ref{fig:kvbar} (bf16,
$s=2$). Dense models cache at every layer; hybrids only at their global-attention layers.}
\label{tab:kv}
\begin{tabular}{@{}lrrrr@{}}
\toprule
Model & \shortstack[r]{global-attn\\layers} & $n_{\mathrm{kv}}$ & $d_{\mathrm{head}}$ & KV / 1M tok \\
\midrule
\multicolumn{5}{@{}l}{\emph{dense GQA ($\rho{=}1$, full $\Theta(L)$ KV)}}\\
Llama-3.1-70B               & 80 & 8  & 128 & ${\approx}\,305$ GiB \\
Qwen3-32B                   & 64 & 8  & 128 & ${\approx}\,244$ GiB \\
Granite-4.0-Micro (3B)      & 40 & 8  & 64  & ${\approx}\,76$ GiB \\
\midrule
\multicolumn{5}{@{}l}{\emph{MLA compressed latent ($\rho{=}1$)}}\\
DeepSeek-V3                 & \multicolumn{3}{c}{compressed latent ($512{+}64$)} & ${\approx}\,65$ GiB \\
LongCat-Flash               & \multicolumn{3}{c}{compressed latent} & ${\approx}\,60$ GiB \\
\midrule
\multicolumn{5}{@{}l}{\emph{layer-wise / SWA hybrid (KV at global layers only)}}\\
Command~A ($3{:}1$ SWA)     & 16 & 8  & 128 & ${\approx}\,61$ GiB \\
MiniMax-01 (linear $1{:}7$) & 10 & 8  & 128 & ${\approx}\,38$ GiB \\
Nemotron-H-56B (Mamba-2)    & 10 & 8  & 128 & ${\approx}\,38$ GiB \\
Jamba-1.5-Large ($1{:}7$)   & 9  & 8  & 128 & ${\approx}\,34$ GiB \\
Qwen3-Next-80B (GDN $3{:}1$) & 12 & 2  & 256 & ${\approx}\,23$ GiB \\
Granite-4.0-H-Small ($9{:}1$) & 4 & 8 & 128 & ${\approx}\,15$ GiB \\
Kimi-Linear-48B (KDA$+$MLA) & \multicolumn{3}{c}{compressed latent} & ${\approx}\,7$ GiB \\
Phi-4-mini-flash (SambaY)   & 1  & 20 & 64  & ${\approx}\,5$ GiB \\
YOCO-3B (single cache)      & 1  & 8  & 128 & ${\approx}\,4$ GiB \\
Jet-Nemotron-2B (JetBlock)  & 2  & 2  & 128 & ${\approx}\,1.9$ GiB \\
\bottomrule
\end{tabular}
\end{table}

\noindent The ${\approx}160\times$ span (Llama-3.1-70B $305$\,GiB to Jet-Nemotron-2B $1.9$\,GiB) is
dominated by the collapse of $n_{\mathrm{layers}}^{\mathrm{attn}}$ (and MLA latent
compression), the memory-side content of R1; five headline figures were re-derived from
primary configs to within $2.4\%$ of the labs' values.

\section{Architecture Census}
\label{app:census}

\begin{figure}[t]
  \centering
  \includegraphics[width=\linewidth]{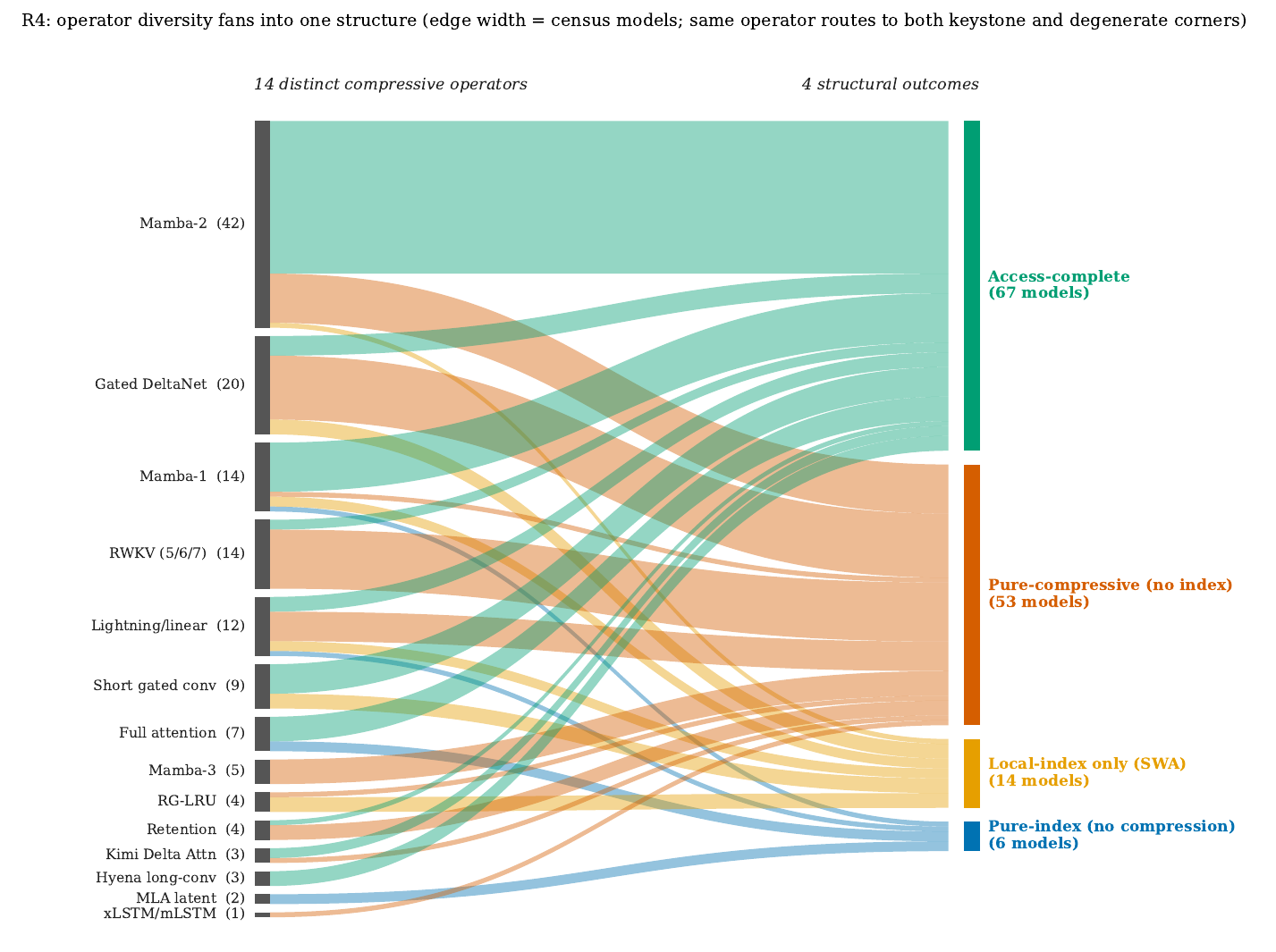}
  \caption{R4 made concrete: fourteen distinct compressive operators (left) route to four
  structural outcomes (right); the same operator family appears in both the access-complete
  keystone (green) and the degenerate corners, so membership is set by pairing with a
  scalable global index, not by operator. Edge width counts census models ($N{=}140$).}
  \label{fig:sankey}
\end{figure}

The Section~\ref{sec:regularities} statistics come from a census of $N=140$ released
architectures (2023-07 to 2026-05; $147$ raw records, $7$ duplicates merged), from
primary-source configs, papers, and model cards ($90\%$ from a primary \texttt{config.json}
or explicit paper statement, the rest flagged lower-confidence). Parsing rule: $\rho$ is
global-attention over total layers, read from the layer map (block patterns expanded to
depth; config-vs-paper conflicts resolved for the config and flagged). Sixty-one models
carry $\rho \in (0,1)$: mean $0.207$, median $0.167$, sd $0.135$, range $[0.031, 0.500]$,
with $57\%$ in $[1/12, 1/4]$, a $3.4\times$ enrichment over uniform. The band boundaries are
round design fractions chosen after inspecting the data, so the enrichment multiples are
descriptive, not significance tests (the going-forward test is
Prediction~\ref{pred:band}). Per-cohort dispersion trends down
($\sigma = 0.148/0.119/0.067$ over 2024/25/26; cohort $n = 23/31/6$, and five of the six
2026 hybrids sit at exactly $\rho = 1/4$) but not significantly ($p\approx0.06$), so
we treat the cross-sectional concentration as the regularity and the narrowing as
suggestive (Figure~\ref{fig:rho}).

\paragraph{$\rho$-semantics a careful reading must separate.} (i) Parallel-hybrid and
hybrid-head models (Falcon-H1, Hymba) run attention in every layer and are not directly
comparable to layer-wise $\rho$. (ii) Sliding-window-only ``hybrids'' (Samba, Based) hold
no scalable global index and are recorded at $\rho{=}0$. (iii) MLA flagships (DeepSeek-V3,
LongCat-Flash) sit at $\rho{=}1$: they cheapen the index by compression, not layer share,
and hold no $O(1)$-state channel, hence not access-complete. (iv) Single-shared-cache
designs (YOCO, SambaY) retain genuine global access, placed below the band by the
stored-cache convention. The above-band tail is distillation sweeps (MambaInLlama at
$\rho \in \{1/8, 1/4, 1/2\}$) and edge models (LFM2); the below-band tail, aggressive 2025
production hybrids (Nemotron-H, Hunyuan-TurboS). Under the explicit frontier rule of
\texttt{experiments/census\_frontier.py} (dropping the 10 distillation conversions and the 6
LFM2 edge models, $n{=}45$), the concentration tightens relative to the full set: median
$\rho = 0.125$, $\sigma = 0.109$ (vs.\ $0.135$), $62\%$ in band (vs.\ $57\%$). The
per-model table, every source URL, and the adversarial re-derivations are in the
supplementary census report.

\paragraph{Robustness of the concentration} (Figure~\ref{fig:robustness}).
\emph{(i)~Operationalization-agnosticism:} the two share conventions,
$\rho_{\mathrm{layer}} = g/L_{\mathrm{total}}$ and $\rho_{\mathrm{mix}} = g/L_{\mathrm{mix}}$
(mixing layers only), rank-correlate at Spearman $\rho_s = 0.90$ across the $60$ hybrids
with full layer maps, and both track the absolute KV-per-token axis
($\rho_s = 0.36$--$0.46$), so placement is not a layer-counting artifact.
\emph{(ii)~Statistical solidity:} a $10$k bootstrap puts the uniform-null enrichment at
$3.4\times$, $95\%$ CI $[2.7, 4.1]$. \emph{(iii)~Within-family:} the in-band fraction holds
across families (Mamba-2 $60\%$, Mamba-1 $80\%$, Gated DeltaNet $67\%$; the lone outlier is
the LFM2 edge line, $25\%$). \emph{(iv)~Null sensitivity:} against a design-menu null
(reduced fractions, denominator $\le 12$) the enrichment is $2.35\times$ ($3.3\times$ at
$\le 32$); against a scale-invariant null on $[0.031, 0.5]$ it falls to $1.45\times$ (CI
$[1.1, 1.7]$). The concentration survives the design-menu nulls but weakens materially
under the scale-invariant one, so R1 leans on cross-sectional stability, within-family
replication, and frontier tightening, not the enrichment multiple alone.
\emph{(v)~Pseudo-replication:} a cluster bootstrap over the 27 mechanism-family clusters
gives in-band fraction $0.574$, CI $[0.41, 0.74]$, enrichment $3.4\times$ (CI $[2.4,
4.5]$), unweighted cluster mean $0.53$; organization-level clusters (27 organizations)
give the same intervals to two decimals. Under the most conservative combination (cluster
resampling $\times$ scale-invariant null) the CI is $[1.03, 1.87]$, barely excluding
$1.0$; we state this rather than hide it.

\begin{figure}[t]
  \centering
  \includegraphics[width=0.9\linewidth]{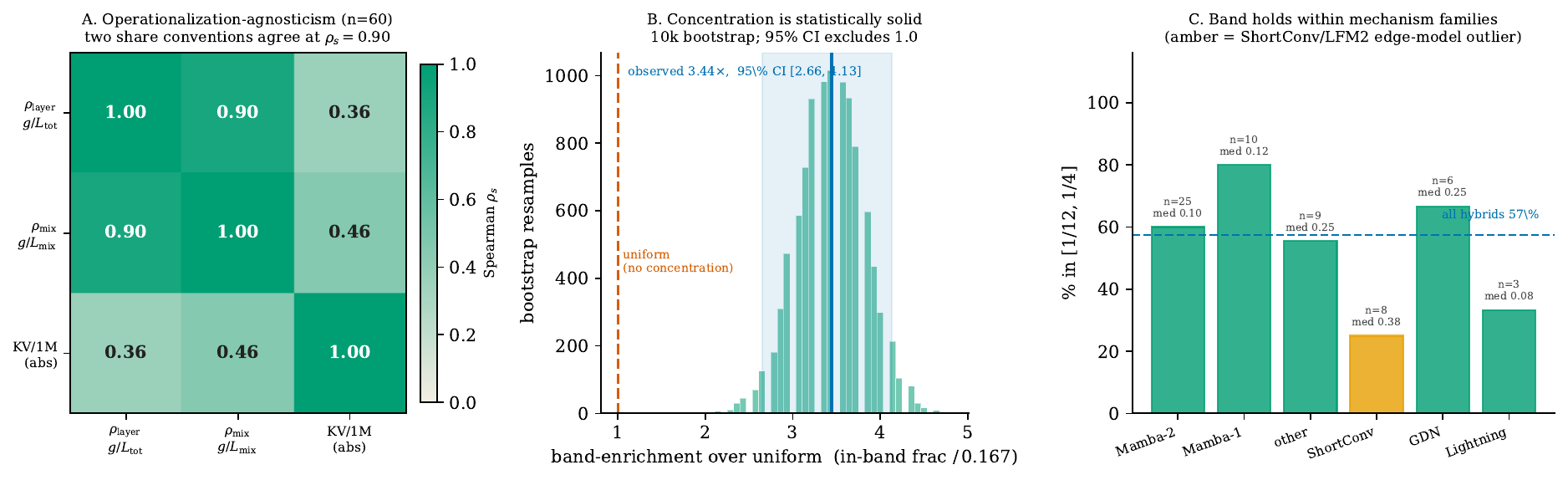}
  \caption{Robustness of the $\rho$ concentration, in the spirit of PRH's metric-agnosticism check
  (its Figure~12). \textbf{A:} Spearman rank-correlation among operationalizations of index
  reliance; the two share conventions ($\rho_{\mathrm{layer}}$, $\rho_{\mathrm{mix}}$) agree at
  $0.90$, and both track the absolute-KV economic axis. \textbf{B:} a $10$k bootstrap of the
  band-enrichment over a uniform spread; the $95\%$ CI $[2.7, 4.1]$ excludes $1.0$. \textbf{C:} the
  in-band fraction holds within mechanism families (amber $=$ the ShortConv/LFM2 edge-model
  outlier). Computed from the census; no schematic content.}
  \label{fig:robustness}
\end{figure}

\paragraph{The access spectrum, visualized.} A deterministic PCA of seven architectural
properties ($\rho$, $\log$ KV/token, $\log$ parameters, depth, SSM-layer share, $\log$
index-width, $\log$ context) places the census on an access-and-scale axis (PC1, $46\%$
variance) correlating with $\rho$ alone at only $0.66$, so it is multi-feature
(Figure~\ref{fig:embed}). Because the features encode access structure, this is a
visualization, not independent evidence: the pure-compressive ($\rho{=}0$) and pure-index
($\rho{=}1$) families flank the access-complete hybrids as an ordered spectrum (hybrids
overlap both corners at the edges, as a spectrum rather than a partition predicts), the
data realization of Figure~\ref{fig:shannon}.

\begin{figure}[t]
  \centering
  \includegraphics[width=0.9\linewidth]{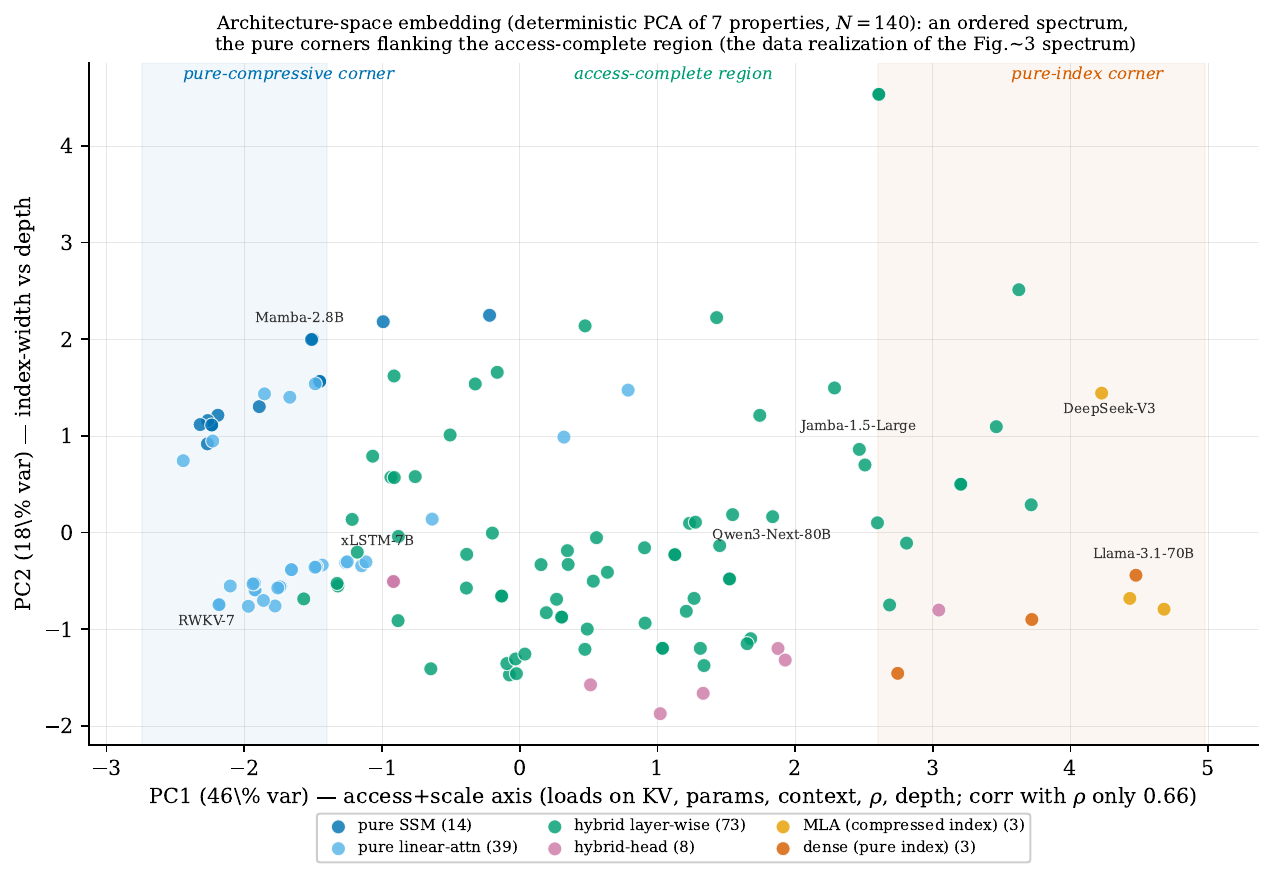}
  \caption{Architecture-space embedding of the census (deterministic PCA of seven standardized
  architectural properties; PRH's Figure~2-right analogue). PC1 ($46\%$ variance) is an
  access-and-scale axis loading jointly on KV/token, parameters, context, $\rho$, and depth;
  its correlation with $\rho$ alone is only $0.66$. The pure-compressive ($\rho{=}0$) and
  pure-index ($\rho{=}1$) corners flank the access-complete hybrids; PC2 ($18\%$) contrasts
  index-width with depth. Measured counterpart of Figure~\ref{fig:shannon}.}
  \label{fig:embed}
\end{figure}

\providecommand{\ding}[1]{\ifnum#1=51\ensuremath{\checkmark}\else\ensuremath{\times}\fi}

\section{Long-Context Capability Data}
\label{app:capdata}

This appendix gives the per-model data behind Figure~\ref{fig:capability}. We stress the
metric-heterogeneity caveat: RULER scores, passkey/needle lengths, and ``none published'' are not
strictly commensurable and must not be read as a single ranking; the robust content is the
\emph{qualitative} separation by access class. RULER is reported on a $0$--$100$ scale (MiniMax-01's
native $0$--$1$ score is rescaled), with the evaluation length annotated where the source fixes one;
``avg'' denotes a RULER average across lengths. The blank cells in the no-index classes are a
\emph{reporting-practice} observation, not negative measurements: those families largely do not
publish a long-context retrieval number, which may reflect target use cases or benchmark timing
rather than capability.

\begin{table}[h]
\centering
\footnotesize
\caption{Long-context retrieval by access class (primary-source-verified). \ding{51}~$=$ holds a
scalable global index; \ding{55}~$=$ does not. MMLU is the short-context parity check.}
\label{tab:capdata}
\begin{tabular}{@{}lrll@{}}
\toprule
Model & $P$ (B) & Long-context retrieval & MMLU \\
\midrule
\multicolumn{4}{@{}l}{\emph{Pure SSM (no scalable global index)} \;[\ding{55}]}\\
\quad Mamba-2.8B & 2.8 & \textit{none published} & -- \\
\quad Falcon-Mamba-7B & 7.27 & \textit{none published} & 62.1 \\
\midrule
\multicolumn{4}{@{}l}{\emph{Pure linear-attention (no scalable global index)} \;[\ding{55}]}\\
\quad xLSTM-7B & 7 & RULER 20@131K & 58.8 \\
\quad RWKV-7 (2.9B) & 2.9 & passkey 35K & 55.0 \\
\midrule
\multicolumn{4}{@{}l}{\emph{Sliding-window-only (local index only)} \;[\ding{55}]}\\
\quad RecurrentGemma-9B & 9 & \textit{none published} & 60.5 \\
\quad Samba-3.8B & 3.8 & \textit{none published} & 71.2 \\
\midrule
\multicolumn{4}{@{}l}{\emph{Access-complete hybrid (compressive state + scalable global index)} \;[\ding{51}]}\\
\quad MiniMax-Text-01 & 456 & RULER 94.7@128K & 88.5 \\
\quad Jamba-1.5-Large & 94 & RULER 93.9@256K & 80.0 \\
\quad Qwen3-Next-80B & 80 & RULER 91.8 (avg) & 80.6 \\
\quad Jamba-1.5-Mini & 52 & RULER 86.1@256K & 69.7 \\
\quad Nemotron-Nano-12B & 12 & RULER 84.7@128K & 78.2 \\
\quad Command A & 111 & RULER 84.6@256K & 85.5 \\
\quad Kimi-Linear-48B & 48 & RULER 84.3@128K & 72.7 \\
\quad Bamba-9B-v2 & 9.78 & passkey 33K & 67.9 \\
\quad Zamba2-7B & 7.4 & passkey 16K & 67.2 \\
\quad Nemotron-H-56B & 56 & \textit{none published} & 84.2 \\
\quad Granite-4.0-H-Small & 32 & \textit{none published} & 78.4 \\
\quad Hunyuan-TurboS & 56 & \textit{none published} & 87.9 \\
\quad Falcon-H1-7B & 7 & \textit{none published} & 76.8 \\
\quad Hymba-1.5B & 1.5 & NIAH (qual.) & 51.2 \\
\midrule
\multicolumn{4}{@{}l}{\emph{MLA (compressed global index, $\rho{=}1$)} \;[\ding{51}]}\\
\quad DeepSeek-V3 & 671 & NIAH (qual.) & 88.5 \\
\midrule
\multicolumn{4}{@{}l}{\emph{Dense (full $\Theta(L)$ index, $\rho{=}1$)} \;[\ding{51}]}\\
\quad Qwen3-32B & 32.8 & RULER 93.7 (avg) & 83.6 \\
\quad Phi-3-mini-128K-instruct & 3.8 & RULER 84.6 (avg) & 69.7 \\
\quad Llama-3.1-8B & 8 & RULER 77@128K & 69.4 \\
\bottomrule
\end{tabular}
\end{table}

\begin{table}[h]
\centering
\footnotesize
\caption{The controlled Mamba2InLlama distillation staircase~\citep{wang2024mambainllama}: a fixed
Llama-3-8B-Instruct base distilled into hybrids retaining different fractions of full-attention
layers. Capability rises monotonically with the retained attention fraction $\rho$ on every metric;
AlpacaEval-LC at $\rho{=}1/2$ ($26.78$) exceeds the dense teacher ($22.90$). This is a
same-base correlation, sharper than the cross-model Table~\ref{tab:capdata} but not a clean
intervention on access alone: retained attention, inherited teacher ability, and trainable
parameters co-vary under the conversion recipe.}
\label{tab:staircase}
\begin{tabular}{@{}lrrr@{}}
\toprule
retained attn.\ $\rho$ & MT-Bench & AlpacaEval-LC & MMLU (0-shot) \\
\midrule
$0$ & 5.64 & 14.49 & 45.19 \\
$1/8$ & 6.48 & 20.25 & 50.78 \\
$1/4$ & 6.74 & 22.75 & 53.71 \\
$1/2$ & 7.32 & 26.78 & 55.7 \\
\midrule
\emph{dense teacher} ($\rho{=}1$) & 8.00 & 22.90 & ${\sim}68$ \\
\bottomrule
\end{tabular}
\end{table}

\section{Demarcation from Prior Work}
\label{sec:related}

Every brick of CCH is prior work; what is new is the framework organizing them into
``capability convergence.''

\paragraph{Relation to PRH.} CCH is PRH's sequel and boundary. It uses representational
convergence~\citep{huh2024platonic,jha2025vec2vec} in exactly two places (the
representational isomorphism of the operators~\citep{dao2024ssd}; the metric for
Prediction~\ref{pred:comm}), and none of its lower bounds depend on PRH holding.
\citet{convergence2026understanding} found the crack (representations align while
computation diverges), independently confirming CCH's premise; \citet{koepke2026cave}
narrows PRH's cross-modal evidence at scale, reinforcing that premise while leaving CCH's
bounds untouched. The relation is also one of evidential standards: PRH's lineage is
observational (released checkpoints, no pre-registration, no falsification scorecard);
CCH's portfolio is a superset, the same observational layers plus formal lower bounds plus
pre-registered interventions with a timestamped amendment log and a failed prediction
reported as such. We note this to locate the genre's evidential bar, not to diminish our
template.

\paragraph{Relation to the lower-bound literature.} The pure-family walls appear scattered
across prior work: SSM retrieval/copying bounds~\citep{jelassi2024repeat}, recall
measurement~\citep{arora2023zoology}, state-tracking limits via circuit
complexity~\citep{merrill2024illusion}, formal-language
expressivity~\citep{sarrof2024expressive}, the $\mathrm{TC}^0$ bound on
attention~\citep{merrill2023parallelism}, chain-of-thought for serial
problems~\citep{li2024chain}, recency limits~\citep{jelassi2025recency}. Each holds one
brick; CCH's contribution is not a new wall but weaving them into one picture plus an
economic regularity, upgrading ``pure architectures each have limits'' into the directional
hypothesis. The nearest empirical neighbour~\citep{pantazopoulos2026retrievit} characterizes
the stratification on synthetic tasks without naming a convergence hypothesis, an ally
supplying data where CCH supplies theory.

\paragraph{Relation to hybrid engineering.} A large body has demonstrated hybrid
efficiency--accuracy
advantages~\citep{lieber2024jamba,blakeman2025nemotronh,de2024griffin,kimi2025linear,arora2024based,glorioso2024zamba,glorioso2024zamba2},
with systematic ablations of ratio and placement now at controlled
scale~\citep{bae2025hybrid}; CCH differs in asking \emph{why} capability converges to the
hybrid and what the attractor class is, not how a given hybrid performs.

\paragraph{Relation to adjacent hybrid analyses.} The closest work articulates the
two-channel division of labour CCH formalizes: \citet{dong2024hymba} state it directly
(attention for high-resolution recall, SSM heads for efficient summarization);
\citet{ben2024expansion} draw the same line as ``fading'' versus ``eidetic'' memory;
\citet{priming2026hybrid} convert pre-trained Transformers into hybrids cheaply, supporting
R1's economic account. One priority statement: \citet{reasoning2026primitives} had already
identified recall and state tracking as complementary primitives and proposed their
composition (an address produced by serial state transitions, then used for retrieval) as
the regime where hybrids should matter. The composite primitive is theirs; Experiment~E
adds the capacity/complexity framing, pre-registered per-arm wall-based predictions, and
the reliability and ablation measurements. Their headline finding, no uniform hybrid edge
once reasoning tokens are enabled, we cite as the strongest counter-evidence on the CoT
axis (AV2), noting that it is reasoning-augmented and mostly outside $\Rgm$, while their
division-of-labour split bears on the two-channel mechanism. It would falsify CCH if a
pure family matched the hybrid on the witness itself, under $\Rgm$, with length
generalization: the test Prediction~\ref{pred:floor} makes precise.

In one line: prior work holds the representational side, the individual walls, the
two-channel intuition, and the empirics of hybrids; CCH adds the unifying framework, the
name, the access-structure account, and the witness. Table~\ref{tab:demarcation} summarizes.

\begin{table}[t]
\centering
\small
\caption{Placement of CCH against adjacent lines of work. CCH does not claim the rows on the
left as contributions; it claims the synthesis in the right column.}
\label{tab:demarcation}
\begin{tabular}{@{}>{\raggedright\arraybackslash}p{0.30\linewidth}>{\raggedright\arraybackslash}p{0.33\linewidth}>{\raggedright\arraybackslash}p{0.27\linewidth}@{}}
\toprule
\textbf{Line of work} & \textbf{Representative results} & \textbf{What CCH adds} \\
\midrule
Representational convergence &
PRH; universal embedding geometry; representation/computation gap~\citep{huh2024platonic,jha2025vec2vec,convergence2026understanding} &
Reads the gap as an \emph{access}-structure boundary, not a representation one \\
Pure-family lower bounds &
copying/recall, state-tracking, $\mathrm{TC}^0$, formal-language limits~\citep{jelassi2024repeat,arora2023zoology,merrill2024illusion,merrill2023parallelism,sarrof2024expressive} &
Welds the separate ``walls'' into one channel-capacity picture + a witness \\
Two-channel hybrid analyses &
recall vs summarization heads; fading vs eidetic memory~\citep{dong2024hymba,ben2024expansion,pantazopoulos2026retrievit} &
Turns the intuition into a closure-non-preservation statement \\
Hybrid engineering &
Jamba, Nemotron-H, Griffin, Kimi Linear, Zamba~\citep{lieber2024jamba,blakeman2025nemotronh,de2024griffin,kimi2025linear,glorioso2024zamba} &
Explains \emph{why} the share band is a steady state (R1, P4) \\
Hybrid primitive comparison &
no uniform edge \emph{with} CoT; sequential-vs-flat split~\citep{reasoning2026primitives} &
Mostly outside $\Rgm$; its state/index split \emph{supports} the two-channel account \\
\bottomrule
\end{tabular}
\end{table}


\section{Experimental Details and Full Verdicts}
\label{app:experiments}

This section reports the controlled measurements of
Predictions~\ref{pred:floor}--\ref{pred:comm} plus the model-level dissociation
experiment. Everything operates inside $\Rgm$ (fixed depth, no chain-of-thought, no
tools); within each comparison only the sequence-mixing operator differs, with identical
data, optimizer, and budget. Scale up front: trained models are $\sim$0.2--1.1M parameters,
public checkpoints $\le 3$B; these are mechanism-level tests (existence-level for the
single-axis walls, training-reliability-level for the conjunction), not frontier
demonstrations (Limitation~(v)). The broader stratification has independent
support~\citep{pantazopoulos2026retrievit}.

\paragraph{Pre-registration protocol.}
All support/falsification criteria were frozen in a written protocol
(\texttt{experiments/PREREGISTRATION.md}, 2026-07-10) before the corresponding data
existed: pass thresholds (accuracy $\ge 0.9$; cell pass on $\ge 2/3$ seeds), the single
permitted fitted constant ($p_{\mathrm{eff}}$), and explicit falsification conditions.
Thirteen timestamped amendments (1--11, 9b, 9c) added, before their data: control arms
(depth-advantaged Transformer, local-window hybrid); a $5\times$ budget extension arranged
so every predicted-fail side got at least the predicted-pass budget; uniform dense LM
supervision; two redesigns of confounded conditions (reported alongside the originals); the
follow-up phase (expanded seeds/budgets/instances, first-phase results primary); the
Assumption~\ref{ass:sep} collision control; the conjunction witness (9) with its pilot
record (9b) and attention-arm doubled-budget rule (9c); the P3$'$ re-operationalization
(10); and the $n{=}1000$ flagship resolution (11). All completed cells reported, no
post-hoc exclusions. Availability: the pre-registration document, the registered and
post-registration census data, and the analysis scripts cited in the text
(\texttt{census\_frontier.py}, \texttt{census\_postreg.py},
\texttt{expA\_partial\_correlation.py}, \texttt{measure\_collisions.py}) are included as
ancillary files with the arXiv source of this paper, so the chronology and the census
statistics are auditable from the submission itself; the complete experimental bundle
(per-run records and reproduction scripts for Experiments A--E) is public at
\url{https://github.com/wenhui-ml/Capability-Convergence-Hypothesis} and permanently
archived on Zenodo, concept DOI
\href{https://doi.org/10.5281/zenodo.21714418}{10.5281/zenodo.21714418} (the snapshot
accompanying this revision: 10.5281/zenodo.21714419).

\subsection{P2, measured: the state-tracking bifurcation (Experiment C)}
\label{sec:expC}

\textbf{Setup.} $S_5$ word problem: the label at position $t$ is the running product
(120-way). Streams: \emph{swap} (transposition per token) and \emph{general} (arbitrary
$S_5$ element per token). Train $T\!\le\!40$; evaluate $T\in\{40,64,128,256,512\}$;
``length-generalizes'' $=$ pass at $T{=}256$. Architectures (2 layers,
$d_{\mathrm{model}}{=}128$; 3 seeds per cell, 8 on the critical $n_h{=}4$ general cell):
DeltaNet $\beta\in(0,1)$ (no reflection, the production parameterization, starred),
DeltaNet $\beta\in(0,2)$, DeltaProduct
$n_h\in\{2,4\}$~\citep{grazzi2025deltaproduct}, diagonal SSM, Transformer (2-layer and a
depth-advantaged 8-layer with $4\times$ parameters), and LSTM. Controls are width-matched:
the LSTM's ${\sim}4\times$ parameters bias toward a predicted-pass control and cannot
manufacture the fail verdicts (the diagonal SSM and 8-layer Transformer fail with
\emph{more} parameters than the passing DeltaProduct).

\textbf{Results (Figure~\ref{fig:nhgrid}).} Five of the six pre-registered
predictions held outright; the sixth (the $n_h$ ladder) held under its frozen
criterion but proved seed-brittle when the cell was expanded, and is reported
below as a capacity rather than a trainability result.
(i)~\emph{The starred prediction:} DeltaNet $\beta\in(0,1)$ fails both streams
at $T{=}256$; in fact it stays at chance ($\approx 0.008$) even at the
training length after a $5\times$ budget extension (100k steps): without
reflections the model cannot even fit composed $S_5$, a stronger outcome
than predicted.
(ii)~Unlocking one reflection ($\beta\in(0,2)$, $n_h{=}1$) is alone sufficient
for the swap stream: accuracy $1.0/1.0/0.998$ across seeds at $T{=}256$.
(iii)~On the general stream the ladder appears at $n_h{=}4$, as a capacity result, not a
reliability result. At $T{=}256$, $n_h\le 2$ sits at chance in every seed ($\le 0.022$
across 6 seeds); $n_h{=}4$ solves the task in 3 of 8 seeds ($\approx\!1.0$; one seed
extrapolates perfectly to $T{=}512$). The frozen 3-seed cell passed its criterion, but the
8-seed expansion put the pass fraction at $3/8$, so we do not claim reliable trainability.
The non-passing seeds expose optimization brittleness, not missing capacity: two partial
learners, three never leaving chance at the \emph{training} length after 100k steps, one
fitting $T{=}40$ perfectly ($1.000$) yet collapsing at $T{=}256$ ($0.474$), the Transformer
control's signature. The transition is a step, not a staircase ($n_h{=}2$ gains nothing):
only $n_h{=}4$ can track composed $S_5$ at $6.4\times$ training length, no smaller $n_h$
ever does, and reliable learning at this budget is open. The boundary is registered for
this configuration (two layers, one DeltaProduct front); deeper stacks may realize $S_5$
with smaller per-layer $n_h$~\citep{siems2025deltaproduct}.
(iv)~Both Transformers fail extrapolation: the 2-layer control is at chance
everywhere (consistent with fixed-depth $\mathrm{TC}^0$ limits,
Proposition~\ref{prop:circuit}); the 8-layer control reaches $1.0$ in-length on
swap and collapses to $0.03$ at $T{=}64$, high in-distribution performance
with immediate extrapolation failure, the exact signature the circuit wall
predicts for a fixed-depth parallel architecture.
(v)~The diagonal SSM is at chance in every cell.
(vi)~The LSTM (nonlinear full-rank recurrence) passes everything: accuracy
$1.000$ at every length on both streams, including $T{=}512$. This supports the axis
decomposition rather than any ``SSMs are weak'' reading: expressive recurrence buys
composition, and what the LSTM still lacks is the index capacity of Experiment~B's axis
(it belongs to $\Sch$ and would be barred at load).
An optimization note: at the original 20k budget the LSTM (general stream) and $n_h{=}4$
had not converged; the budget extension was pre-registered before those runs, and every
predicted-fail cell received at least the predicted-pass budget, so no failure above is a
budget artifact.

\subsection{P1, measured: the information floor and the scissors gap
(Experiment B)}
\label{sec:expB}

\textbf{Setup.} Executable $\mathrm{NA}(N,b,B;\kappa)$: $N$ key--value bindings ($b{=}8$
bits) planted in a length-512 stream with semantic-overlap distractors $\kappa$, single
query at the end. Arms: pure-state DeltaNet ($m\in\{64,256,1024\}$ scalars);
sliding-window ($w{=}32$); a hybrid holding the \emph{same} $m{=}64$ state plus one
global-attention layer; the same state plus a window-only layer; and full attention. 3
seeds (6 at the doubled budget); identical budgets within each comparison. Amendment-8
collision arms retrain the hybrid with $c\in\{1,3,7\}$ colliding rebinds per key.

\textbf{Results (Figure~\ref{fig:fano}).}
(i)~\emph{Floor rise:} every pure-state curve rises monotonically in $N$ toward the chance
floor ($0.994/0.977/0.854$ at $N{=}128$ for $m{=}64/256/1024$; chance $0.996$).
(ii)~\emph{Load-ratio collapse:} with the frozen $p_{\mathrm{eff}}{=}0.39$ bits/scalar (fit
on $m{=}256$), the 50\%-crossings land at $x = Nb/(m\,p_{\mathrm{eff}}) \in [0.82, 1.03]$,
$m{=}256$ and $m{=}1024$ collapsing within $0.024$ at shared $x$; the $\pm 0.1$ band is met
by the two larger states, violated (up to $0.15$) by the edge-of-trainability $m{=}64$, so
we report the criterion as partially met.
(iii)~\emph{The scissors gap, measured:} at $N{=}128$ the pure $m{=}64$ model has error
$0.994$ while the hybrid carrying the \emph{same} state plus one global-attention layer
has mean error $0.060$ (seeds $0.018/0.125/0.038$) at the registered 12k budget, one seed
crossing the strict frozen bound ($\le 0.05$) at the last grid point; the amendment-logged
doubling (24k, six seeds) resolves the miss as undertraining, mean error $\le 0.0003$ at
every $N$ (worst seed $0.001$), a measured scissors of $0.994$ vs.\ $0.000$ on the same
state size. The pure-state side sits at chance with no improvement trend across budgets,
and every predicted-fail cell received at least the predicted-pass budget; the ``more
optimization cannot help'' reading is licensed by the effective-capacity fit
($x_{\mathrm{eff}}\!\approx\!16$), not the nominal floor ($x\!\approx\!1$ at bf16, where
the bound is weak).
(iv)~\emph{The index must be global:} the same state paired with a
window-limited layer is not rescued (error $\ge 0.91$ at all $N$); being
access-complete, not merely having attention, is what matters
(cf.\ Figure~\ref{fig:sankey}).
(v)~\emph{Horizon wall:} the sliding-window model is at chance ($0.99+$) for out-of-window
bindings at every $N$; within-window retrieval is good at small $N$ ($0.07$ at $N{=}2$),
but the ``near-perfect independent of $N$'' clause is unrealizable in this design ($>\!8$
bindings no longer fit the window), so we count it partially confirmed.
(vi)~\emph{Fano sanity:} no architecture beats the nominal-capacity floor anywhere.
(vii)~\emph{Compressibility:} with block-duplicated values ($c{=}16$) the 50\%-crossing
shifts to $N\!\approx\!27$, matching the effective-entropy prediction ($\approx\!25$) over
the nominal one ($\approx\!12$); but $c{=}4$ tracks the nominal prediction and the block
design admits a retrieval-free shortcut at $N\le c$, so this axis is mixed evidence.
(viii)~\emph{A null and an anomaly:} the semantic-overlap ordering ($\kappa$) shows no
detectable effect at matched weights ($0.610$ vs.\ $0.609$), and 2-layer full attention is
not the upper bound the schematic assumed ($0.91$ at $N{=}128$, an optimization failure at
this scale). Trainability itself depended on state size: 3 of 6 pure $m{=}64$ runs never
learned even $N{=}2$ (no such failure elsewhere); all means include these seeds, per
protocol. Separating evidential registers: the pure-state failures are consistent with the
capacity framework at \emph{effective} capacity, but the nontrivial nominal floor is not
triggered at these settings (nominal $x\approx1$ at $m{=}64$), so the quantitative backing
is at effective, not nominal, capacity; the full-attention anomaly and the $m{=}64$
convergence failures are optimization-level observations with no capacity claim attached.
(ix)~\emph{Negative control on Assumption~\ref{ass:sep} (amendment 8):} we drive the code
distance to \emph{zero}: every planted key is bound $c{+}1$ times to distinct values,
shuffled, the labeled binding a uniformly random one, so the architecture-independent
Bayes floor is error $c/(c{+}1)$ and log-loss $\log_2(c{+}1)$ bits, independent of $N$. The
trained hybrid lands on the error floor almost exactly at every load: worst-cell deviation
from $0.500/0.750/0.875$ is $0.023/0.025/0.005$ (frozen band $\pm 0.05$; the two mildly
sub-floor cells lie inside fixed-picker fluctuation, checked per the pre-registered bug
clause). Log-loss sits \emph{above} its floor everywhere ($+0.05$ to $+0.89$ bits, growing
with $c$ and $N$; the $\pm 0.4$ band holds at $c{=}1$, is exceeded at $c\in\{3,7\}$): a
calibration gap, not a retrieval one, in the safe direction. Neither pre-registered red
flag fired (no drift toward chance $0.996$; no genuine sub-floor cell). Reading: the
scissors advantage of (iii) is purchased exactly where Assumption~\ref{ass:sep} says, in
write-time separability, and vanishes bit-for-bit when it is removed; the crossing claim's
conditional structure is measurable, and measured.

\subsection{Representational convergence without capability convergence
(Experiment A)}
\label{sec:expA}

\textbf{Setup.} Three families on the same corpus (The Pile) and tokenizer, 70M--3B: Pythia
(attention), Mamba (pure SSM), RWKV-4 (linear RNN). Capability: $N$ templated facts at
controlled depth in Pile filler ($L\in\{896,1900\}$), exact-match retrieval, $200$ instances
per cell. Representation: PRH's mutual $k$-NN alignment ($k{=}10$; CKA as a check) on 1024
shared snippets, max over layer pairs (a caveat: the max favors deeper models, so read
absolute alignments with care; chance $0.0098$, and orderings replicate across all three
pairings).

\textbf{Results (Figure~\ref{fig:dissoc}).}
(i)~\emph{PRH replicates:} cross-family alignment grows with scale (Spearman
$0.57/0.87/0.51$ for the three pairs; $n=30/24/20$, all $p<0.05$), reaching $0.7+$ against
chance $0.01$.
(ii)~\emph{Capability nonetheless stratifies by access structure:} at matched
$\sim$2.8B scale ($N{=}16$ early-depth facts, $L{=}1900$), retrieval accuracy
is Pythia $0.345$, Mamba $0.05$, RWKV $0.00$. The pre-registered margin ($\ge 30$pp over
both recurrent families) is met against RWKV ($34.5$pp) but narrowly missed against Mamba
($29.5$pp); at $n{=}200$ both $95\%$ intervals ($29.5 \pm 7.2$, $34.5 \pm 6.6$pp) straddle
the frozen line. Amendment 11 re-evaluated the six deciding cells at $n{=}1000$, committed
in advance to the point estimate: the flagship cell resolves above the line ($32.0 \pm
3.3$pp vs.\ Mamba, $36.5 \pm 3.0$pp vs.\ RWKV; the near-miss was sampling noise), every
other cell clearly below ($26.0/18.7$pp at $N{=}32/64$ in the 2.8B class;
$14.0$--$15.9$pp throughout the 1.4B class). A-ii stands as partial with the boundary
precisely mapped; the ordering Pythia $>$ Mamba $>$ RWKV is unambiguous in all ten
matched-scale cells.
(iii)~\emph{The failure signature is architectural:} Mamba's accuracy falls with distance
from the query (recency; $0.065\!\to\!0.326$ early-to-late for Mamba-1.4b, same slope at
every size) while Pythia-2.8b is flat ($0.441\!\to\!0.452$) and RWKV sits at a floor.
(iv)~\emph{The dissociation:} across 74 cross-family pairs, representational
alignment does not predict capability parity: the correlation between alignment and the
absolute gap is significantly \emph{positive} (Spearman $+0.48$, $p<10^{-4}$; the frozen
falsification was a significantly negative one), and survives controlling for the pair's
log geometric-mean parameter count (partial Spearman $+0.30$, $p=0.009$;
\texttt{experiments/expA\_partial\_correlation.py}), so it is not a scale
artifact. Representational convergence does not purchase capability convergence; the
boundary tracks access structure, the measured form of the paper's central claim.

\begin{figure}[tp]
  \centering
  \includegraphics[width=\linewidth]{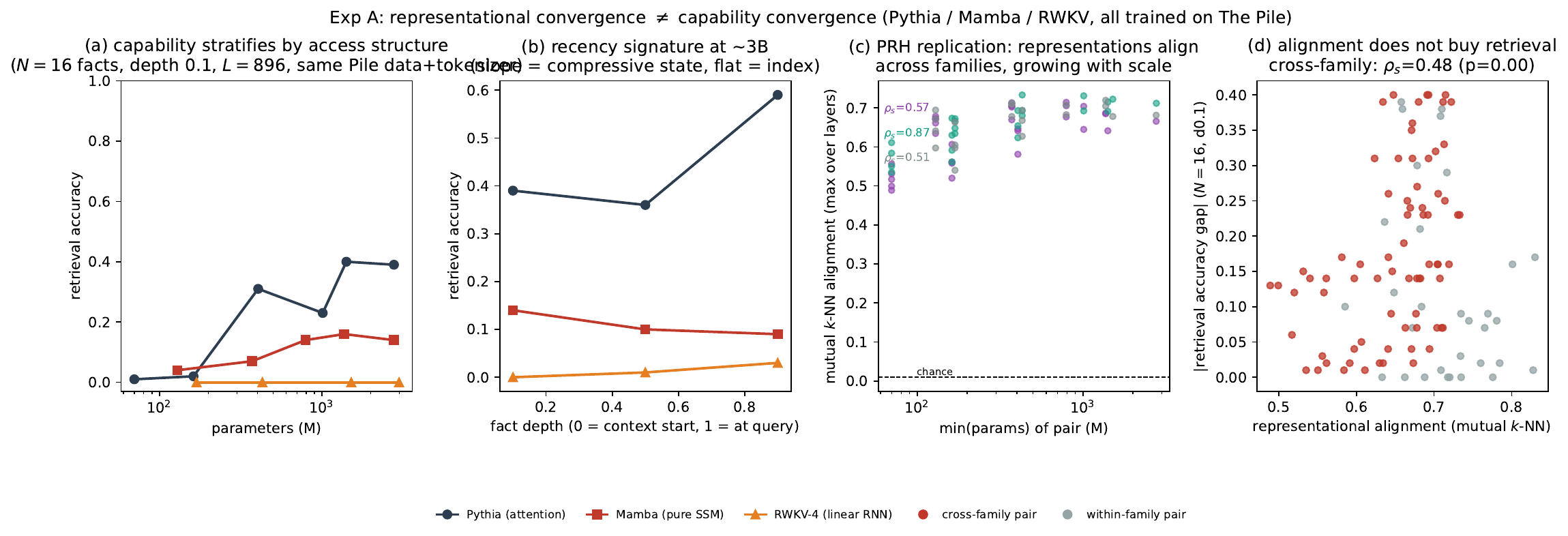}
  \caption{Representational convergence $\ne$ capability convergence
  (Experiment~A: Pythia / Mamba / RWKV-4, all pretrained on The Pile with a
  shared tokenizer). \textbf{(a)} Retrieval capability stratifies by access
  structure at matched scale and data. \textbf{(b)} The failure signature at
  $\sim$3B: recency slope for the compressive-state family, near-flat for the
  indexed family, floor for the linear RNN. \textbf{(c)} PRH replication:
  cross-family mutual $k$-NN alignment grows with scale (all three family
  pairings $p<0.05$). \textbf{(d)} The dissociation: alignment does not
  predict the capability gap (cross-family Spearman $\rho=+0.48$,
  positive rather than negative); the high-alignment/large-gap quadrant is occupied.
  \emph{Measured data.}}
  \label{fig:dissoc}
\end{figure}

\subsection{P3 miniature: a failed prediction, reported as such
(Experiment D)}
\label{sec:expD}

We tested Prediction~\ref{pred:comm} in miniature on Experiment~B's trained hybrids:
cross-channel alignment for solving finals vs.\ failing checkpoints vs.\ random
initialization. \textbf{The prediction failed, direction reversed, and the reversal
survived every control.} Mutual $k$-NN group means order strictly, reversed
from the prediction: random initialization $0.576$ (the
ceiling, not the floor) $>$ undertrained $0.527$ $>$ failing finals
$0.376$ $>$ solving finals $0.284$. Completing the protocol strengthened
the reversal: under linear CKA the ordering is more extreme (random $0.993$; solving
$0.13$--$0.33$). The reversal also holds across training budget and arms (mutual $k$-NN):
doubled-budget hybrids align lower still ($0.231$ at 24k vs.\ $0.284$ at 12k), and the
registered $\beta\in(0,1)$ ``failing arm'' does not fail this task ($\beta\in(0,1)$
breaks tracking, not retrieval), its solving checkpoints showing the same low $k$-NN
alignment ($0.261$). Two lessons: random initialization is not a low-alignment floor (two random
channels reading the same tokens echo the same input geometry), and in solving hybrids
training \emph{differentiates} the channels, consistent with complementary loads but
opposite to the registered signature. As committed, this is a failed prediction of CCH;
Figure~\ref{fig:comm} retains its schematic label with the failure noted.

\begin{figure}[t]
  \centering
  \includegraphics[width=0.72\linewidth]{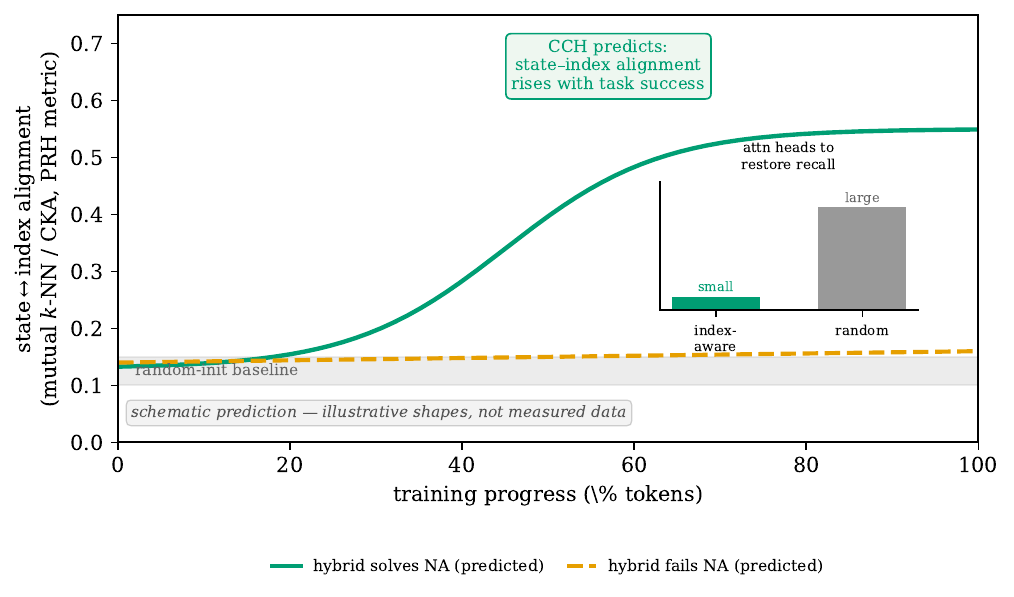}
  \caption{Channel commensurability (Prediction~\ref{pred:comm}): solving hybrids were
  predicted to align across channels above random initialization. \emph{Schematic.} The
  pre-registered test \emph{failed, direction reversed} (Appendix~\ref{sec:expD}): solving
  hybrids align \emph{below} random initialization, which is itself high because two random
  channels echo the same input geometry; the schematic is retained as a record of the
  original claim.}
  \label{fig:comm}
\end{figure}

\subsection{The conjunction witness (Experiment E) and channel complementarity (D-ii)}
\label{sec:expE}

\textbf{Setup (amendments 9/9b).} Composite task: $N$ bindings [permutation-address $\to$
8-bit value] written first; a swap stream of length $T$ drives the referent through $S_5$;
the query asks for the value at the \emph{composed} final referent. Supervision (uniform
across arms): LM over the stream, dense ``running answer'' labels at instruction positions,
and a weighted answer loss. Train $N \le 64$, $T \le 40$; evaluate
$T \in \{40,64,128,256\}\times N \in \{4,16,64\}$; pass $= 0.9$, cell pass ${\ge}2/3$ of 3
seeds; 120k steps. Arms (frozen from the pilot, 9b): recurrent-front hybrid (two
$\beta\in(0,2)$ DeltaNet layers + one global-attention layer), pure recurrence (three
DeltaNet layers, more total state), 8-layer Transformer, $\beta\in(0,1)$-front hybrid, and
window-32 hybrid.

\textbf{Results} (Figure~\ref{fig:witness}). The witness criterion is met at seven cells,
weighted unevenly. The
axis-supported cell is the predicted long-extrapolation cell ($T{=}256$, $N{=}16$: hybrid
$0.966/0.979/0.964$; no other arm meets the $2/3$-seed criterion). One pure-recurrence seed
solves it ($0.995$), so the cell separates training reliability, not expressivity, as
theory requires ($N{=}16$, $128$ bits, is far inside pure-state capacity); existence-level
separation would need supra-capacity loads, where the available measurement is graded
per-seed dominance at $N{=}64$ (hybrid worst $0.706$ $>$ pure best $0.654$, both under the
pass line). The six remaining cells ($T\le 128$) pass formally but owe substantially to
controls failing to train. Per-arm means at $T{=}256$, $N{=}\{4,16,64\}$: hybrid
$0.996/0.970/0.801$; pure recurrence $0.721/0.548/0.269$ (per-seed at $N{=}64$:
$0.654/0.128/0.023$; perfect $T$-extrapolation in converged seeds with monotone
$N$-degradation whose onset, at nominal $x \approx 0.2$, sits far below the naive floor,
an effective-capacity observation under dual duty, flagged as open); 8-layer Transformer
$0.253/0.065/0.017$ (never fits even in-length, nor at the doubled 240k budget of
amendment 9c, where in-length accuracy stayed at the $1/N$ baseline, $0.27$ at $T{=}40$,
$N{=}4$: not a budget artifact, E-iii holding only in its fail half); $\beta\in(0,1)$
hybrid $0.424/0.122/0.039$ (both axes lost at once, retrieval inheriting the tracking
failure through the composed address; the sub-prediction that it would keep in-length
retrieval was wrong, instructively, since the axes are coupled by design); window-32
hybrid $0.668/0.275/0.044$ (passes only where its internal state carries the load). Graded
margins (hybrid minus best other arm) at $T{=}256$: $+0.275 / +0.422 / +0.532$ for
$N{=}4/16/64$. Notes: 3 of 15 runs are convergence failures (included per protocol); the
pilot record (9b) documents that single-recurrent-layer hybrids lose tracking to an
in-length attention shortcut and that weight tying harms.

\textbf{D-ii: channel complementarity (P3$'$, amendment 10).} On the three solving
checkpoints, channels identified by mixer type: (a)~\emph{lesions}: bypassing either
channel collapses every cell to ${\le}0.03$ in every seed (drops $\ge 0.96$), so the
learned solution is irreducibly two-channel; (b)~\emph{probes}: the answer is linearly
decodable from the index mixer at $1.0$ vs.\ $0.004$ from the state mixer (mean over
seeds; min $0.003$), localizing
retrieval where the mechanism says, matching independent production-hybrid ablations where
attention removal collapses retrieval with no SSM compensation~\citep{michalak2025someattention}.
The frozen differential sub-criteria failed because the conjunction couples the axes (either
lesion kills the pipeline; the referent is readable from the last recurrent mixer in only 1
of 3 seeds), so joint necessity is partly built in by the coupling and the discriminative
weight falls on the localization probe. Scope note: lesioning at evaluation is a
distribution shift, so lesions establish that the \emph{learned solution} uses both
channels, not that the task requires both modules; the retrained-control version of the
stronger claim is the arm comparison above. P3$'$: partially established (joint necessity).

\subsection{A first corpus measurement of write-time key separability (exploratory)}
\label{sec:collnat}

This exploratory, post-hoc measurement runs a first surface-form version of the
Section~\ref{sec:action} protocol for how often natural corpora supply
Assumption~\ref{ass:sep}'s write-time separability; we report it as scoping, not a verdict.
Windows are write contexts (a Pile document, or a Python file from $1{,}403$ numpy/scipy
files, also pooled $8\times$); keys are entity-like capitalized spans and numeric IDs
($\ge 4$ digits) in text and AST-extracted identifiers in code; two distinct keys
\emph{near-collide} at character-trigram Jaccard $\ge 0.5$ (numeric IDs: equal length,
Hamming $\le 1$). \textbf{Results:} entity-like keys are largely separable ($87.5\%$ at
document scale, $86.5\%$ pooled), numeric IDs are not ($50.7\%$, $46.8\%$), code
identifiers sit between ($56.2\%$, $54.6\%$). So Assumption~\ref{ass:sep} is plausible for
entity-style keys and materially strained for morphologically regular ones, exactly the
loads where the collision control predicts the hybrid advantage shrinks toward the
ambiguity floor. The proxy is first-order in both directions (a surface near-neighbor need
not collide in a learned embedding; semantic aliasing is invisible to it); the
embedding-level version is open. Script and raw results in
\texttt{experiments/collision\_naturalness/}.

\subsection{Deviations, and what these results do not establish}
\label{sec:exp-honesty}

\textbf{Deviations from the frozen protocol} (beyond the thirteen timestamped
amendments): the sliding-window width was $w{=}32$, not the registered
$w{=}64$ (the receptive field remains far smaller than the binding--query
distance, which is what the horizon-wall logic requires); Experiment~A used
$L\in\{896,1900\}$ rather than $2048$ (context-window limits of the public
checkpoints), and its first pass used $100$ instances per cell, brought to
the registered $\ge 200$ by the follow-up amendment with unchanged
conclusions; the exact analytic entropy of the compressibility conditions was
approximated by $b-\log_2 c$. The $\beta\in(0,1)$-hybrid arm of Experiment~D,
omitted in the first pass, was completed under the same amendment. None of
these was decided after seeing results it could affect.

\textbf{Scorecard.} Of 19 pre-registered predictions, under the frozen first-phase
criteria: 11 supported, 7 partial (direction confirmed, a threshold or sub-clause missed),
1 failed (D-i, direction reversed); the follow-up amendment moved one prediction in each
direction, leaving the counts unchanged. The conjunction-witness and complementarity arms
(9/9b/10) are scored separately: the witness criterion is met; of five per-arm
predictions, three held, one held in its fail half only (the 8-layer control), and one
sub-prediction was wrong in an instructive direction (the $\beta\in(0,1)$ arm); P3$'$ is
partially established. No pre-registered \emph{falsification} condition for P1 or P2 was
triggered. Because the $11/7/1$ count depends on how sub-predictions were split, it should
not be read as summary evidence; at major-prediction granularity: P1 partial (scissors
clean; capacity calibration not independent), P2 partial (existence endpoint clear;
reliable trainability open), P3 failed, P4 suggestive only ($p\approx0.06$), P5 untested.
The per-prediction detail, not either count, is the canonical statement.

\textbf{Scope.} These experiments establish that the predicted mechanisms (the floor, the
scissors, the bifurcation, the dissociation) are real and land where CCH says, at small
scale, under gradient descent rather than by construction (Limitation~(iv)); and, via the
collision control, that the separability assumption is a measurable hinge: remove it and
the hybrid sits on the exact information-theoretic floor. They do not establish frontier
magnitudes (Limitation~(v)), leave P3 without positive evidence, leave $n_h{=}4$
trainability open ($3/8$), and two comparisons (the Pythia--Mamba margin; the $\pm 0.1$
collapse band at $m{=}64$) missed their frozen thresholds and are claimed as directional
only.


\bibliographystyle{plainnat}
\bibliography{cch}

\end{document}